%% file: paper.tex
\documentclass[letterpaper]{article}
\pdfoutput=1
\usepackage{aaai} 
\usepackage{times}
\usepackage{helvet}
\usepackage{courier} 
 \usepackage{xspace}   
\usepackage{amsmath}
\usepackage{latexsym}
\usepackage{amsthm} 
\usepackage{algorithm}
\usepackage[noend]{algpseudocode}
\usepackage{thmtools,thm-restate}
\usepackage{tikz}
\usepackage{macros}
\usetikzlibrary{matrix}

\declaretheorem{theorem}
\declaretheorem[sibling=theorem]{proposition}
\declaretheorem[sibling=theorem]{lemma}
\declaretheorem[sibling=theorem]{corollary}
\declaretheorem[style=definition,sibling=theorem,qed=$\diamond$]{definition}
\declaretheorem[style=definition,sibling=theorem,qed=$\diamond$]{example}
  
\newfloat{procedure}{thp}{lof}
\floatname{procedure}{Procedure}
 
\allowdisplaybreaks
\nocopyright
\begin{document}
%
\title{Datalog Rewritability of Disjunctive Datalog Programs\\ and its Applications to
Ontology Reasoning} 
\author{Mark Kaminski \and Yavor Nenov \and Bernardo Cuenca Grau\\
Department of Computer Science, University of Oxford, UK
}
\maketitle
\begin{abstract}
We study the problem of rewriting a disjunctive datalog
program into plain datalog. We show 
that a disjunctive program is rewritable if and only if 
it is equivalent to a linear disjunctive program, thus
providing a novel characterisation of datalog rewritability.
Motivated by this result, we propose 
\emph{$\semi$ disjunctive datalog}---a novel rule-based 
KR language that extends both datalog and linear disjunctive datalog 
and for which reasoning is tractable in data complexity.
We then explore applications of $\semi$ programs to ontology reasoning
and propose a tractable extension of OWL 2 RL with disjunctive axioms. 
Our empirical results suggest that
many non-Horn ontologies can be reduced to $\semi$ programs 
and that query answering over such ontologies using
a datalog engine is feasible in practice.
 \end{abstract}

\input{introduction}
\input{preliminaries}
\input{characterisation}
\input{weaklinearity}
\input{procedure}
\input{applicationDLs}

\input{related}
\input{evaluation}

\input{discussion}

\section*{Acknowledgements} 

This work was supported by the Royal Society, the EPSRC projects
Score!, Exoda, and $\text{MaSI}^3$, and the FP7 project \mbox{OPTIQUE}.

\bibliography{bib,paper}
\bibliographystyle{aaai}

\clearpage
\onecolumn
\appendix
\input{proofs-characterisation}
\input{proofs-weaklinearity}
\input{proofs-procedure}
\input{proofs-DLs}

\end{document}

%% file: introduction.tex
\section{Introduction}
  
Disjunctive datalog, which extends plain datalog
by allowing disjunction in the head of rules,   
is a prominent KR
formalism that has found many applications in the areas of
deductive databases, information integration
and ontological reasoning
\cite{DBLP:journals/tods/EiterGM97,DBLP:journals/csur/DantsinEGV01}.\footnote{
Disjunctive datalog typically allows for  
negation-as-failure, which we don't consider since we
focus on monotonic reasoning.}
Disjunctive datalog is a powerful language, which can 
model incomplete information.
Expressiveness comes, however, at the expense
of computational cost: 
fact entailment
is co-\textsc{NExpTime}-complete in combined complexity and
\coNP-complete w.r.t.\ data \cite{DBLP:journals/tods/EiterGM97}. 
Thus, even
with the development of optimised implementations
\cite{DBLP:journals/tocl/LeonePFEGPS06}, 
robust behaviour of
reasoners in data-intensive applications cannot be guaranteed.
 

Plain datalog 
offers more favourable computational properties  
(\textsc{ExpTime}-completeness in combined complexity and \textsc{PTime}-completeness
w.r.t.\ data) at the expense of a loss in expressive power
\cite{DBLP:journals/csur/DantsinEGV01}. 
Tractability in data complexity is an appealing property
for data-intensive KR; in particular,
the RL profile of the ontology language OWL 2 was 
designed such that
each ontology corresponds to a datalog program \cite{OWL2-profiles}.
Furthermore, datalog programs obtained from RL ontologies
contain rules
of a restricted shape, and
they can be evaluated in polynomial time  also
in \emph{combined complexity}, thus 
providing the ground for 
robust implementations. 
The standardisation
of OWL 2 RL 
has spurred the development
of reasoning
engines within industry and academia,  such as OWLim \cite{bishop2011owlim}, Oracle's Semantic Data Store~\cite{Wu08}, and RDFox~\cite{MotikNPHO14}. 

We study the problem of rewriting a disjunctive
datalog program into an equivalent datalog program (i.e., one that
entails the same facts for every dataset). By computing such rewritings, we
can ensure tractability w.r.t.\ \mbox{data and exploit}
reasoning infrastructure available for datalog.
Not every disjunctive datalog
program is, however, datalog rewritable \cite{Afrati:1995un}. 

Our first contribution 
is a novel characterisation of datalog rewritability
based on \emph{linearity}: a 
restriction that requires each rule to contain 
at most one 
body atom with an IDB predicate (i.e., a predicate occurring in head
position). For plain datalog,
linearity is known to limit the effect of recursion and lead
to reduced data and combined complexity \cite{DBLP:journals/csur/DantsinEGV01}.
For disjunctive programs the effects of the linearity
restriction are, to the best of our knowledge, unknown.
In Section \ref{sec:characterisation},  we show that every linear disjunctive program
can be polynomially transformed into an equivalent datalog
program; conversely, we also provide
a polynomial
transformation from  
datalog  into
linear disjunctive datalog. Thus, linear disjunctive datalog
and datalog have the same computational properties, and 
linearisability of disjunctive programs is equivalent to 
rewritability into
datalog.


Motivated by our characterisation, 
in Section \ref{sec:weak} we propose  \emph{$\semi$ disjunctive
datalog}: a   
rule language  that extends both datalog and
linear disjunctive datalog.
In a $\semi$ (WL for short) program, the linearity requirement is relaxed: 
instead
of applying to all IDB predicates, it applies only
 to those that ``depend'' on a disjunctive rule.
Analogously to linear disjunctive programs,
 WL programs can be polynomially rewritten 
 into datalog. 
Thus, our language  captures
disjunctive information
while leaving the favourable computational properties of datalog intact.

In
Section \ref{sec:procedure}, we propose a 
linearisation procedure 
based on unfolding transformations. Our procedure  picks 
a non-WL rule and a ``culprit'' body atom 
and replaces it with 
WL rules by ``unfolding'' the
selected atom. Our procedure is 
incomplete: 
if it succeeds, it outputs a WL program, which is rewritten into
datalog; if it fails, no conclusion can be drawn.

In Section \ref{sec:DLs}, we focus on ontology reasoning. We propose an extension
of OWL 2 RL with disjunctive axioms such that
each ontology in our extended profile
maps to a WL program. We show that
the resulting  programs can
be evaluated in polynomial time in \emph{combined} complexity; thus,
fact entailment in our
language is no harder than in OWL 2 RL. 
Finally, we argue that the algorithm in \cite{hms07reasoning} can be combined with
our techniques to rewrite a $\SHIQ$ ontology into a plain datalog program.

We have evaluated our techniques on a large
ontology repository.  Our results show that many
non-Horn ontologies can be rewritten into WL programs, and
thus into datalog. We have tested the 
scalability of query answering using our approach, with promising results.
Proofs of our technical results are delegated to the appendix.

%% file: preliminaries.tex
\section{Preliminaries}


We use standard first-order syntax and semantics 
and assume all formulae to be function-free.
%
We assume that equality $\equality$ 
is an 
ordinary predicate  and that every set of
set of formulae 
contains the standard 
explicit axiomatisation of $\equality$ as a congruence
relation for its signature. 
%


 A \emph{fact} is a ground atom and a 
\emph{dataset} is a finite set of facts. 
A \emph{rule} $r$ is a sentence of the form
${\forall \vec{x} \forall \vec{z}.[\varphi(\vec{x},\vec{z})
  \rightarrow \psi(\vec{x})]}$, where tuples of variables $\vec{x}$
and $\vec{z}$ are disjoint, $\varphi(\vec{x},\vec{z})$ is a
conjunction of distinct equality-free atoms, and $\psi(\vec{x})$ is a
disjunction of distinct atoms. Formula $\varphi$ is the \emph{body} of
$r$, and $\psi$ is the \emph{head}. 
Quantifiers in rules are omitted. We assume that rules are \emph{safe},
i.e., all variables in the head occur in the body.  A rule is
\emph{datalog} if $\psi(\vec{x})$ has at most one atom, and it
is \emph{disjunctive} otherwise.  A \emph{program} $\P$ is a finite set of
rules; it is \emph{datalog} if it consists only of
datalog rules, and \emph{disjunctive} otherwise. 
We assume that rules in $\P$ do not share variables.

For convenience, we treat $\top$ and $\bot$ in a non-standard way
as a unary and a nullary predicate, respectively.
Given a program $\P$, $\P_{\top}$ is the program with a rule
$Q(x_1, \ldots, x_n) \rightarrow \top(x_i)$ for each predicate $Q$ in
$\P$ and each $1 \leq i \leq n$, and a rule $\,\to\top(a)$ for each
constant $a$ in $\P$. We assume that $\P_\top\subseteq\P$ and
$\top$ does not occur in head position in $\P\setminus\P_\top$.
We define
$\P_{\bot}$ as consisting of a rule with $\bot$ as body and empty head. We assume
${\P_\bot\subseteq\P}$ and no rule in ${\P \setminus \P_{\bot}}$ has an empty
head or $\bot$ in the body.
Thus, $\P\cup\Dat\models\top(a)$ for every $a$ in $\P\cup\Dat$, and
$\P \cup \Dat$ is unsatisfiable iff $\P \cup \Dat \models \bot$.

Head predicates in $\P\setminus\P_\top$ are 
\emph{intensional} (or \emph{IDB}) in $\P$. All other predicates (including $\top$)
are \emph{extensional} (\emph{EDB}). An atom is
intensional (extensional) if so is its predicate.  A 
rule is \emph{linear} if it has at most one IDB body atom. 
A program $\P$ is linear if all its rules are.
In contrast to KR, in 
logic programming it is often assumed that IDB
predicates~do not occur in datasets.
This assumption can be lifted
(see, e.g.,~\cite{BryEEFGLLPW07}): 
for every $\P$ and IDB
predicate $Q$ in $\P$, let  $Q'$ be a fresh predicate;
the \emph{IDB expansion} $\P^{e}$ of $\P$
is obtained from $\P$ by renaming each IDB
predicate $Q$ in $\P$ with $Q'$ and adding a rule $Q(\ve x)\to Q'(\ve
x)$, with $\ve x$ distinct variables. 
Then, 
for each $\Dat$ and each fact $\alpha$ over the signature of $\P$
we have 
$\P \cup \Dat \models \alpha$ iff $\P^{e} \cup \Dat \models \alpha \theta$,
where $\theta$ is the predicate substitution mapping each IDB predicate $Q$ 
to~$Q'$.

The \emph{evaluation}
of $\P$ over a dataset $\Dat$ is the set $\eval\P\Dat$ which comprises
$\bot$ if $\P \cup \Dat$ is unsatisfiable and all facts
entailed by ${\P \cup \Dat}$ otherwise.  For a set of predicates $S$, 
$\eval\P\Dat|_S$ consists of those facts in $\eval\P\Dat$ involving predicates
in $S\cup\{\bot\}$.
Program $\P'$ is a 
\emph{rewriting} of $\P$ w.r.t.\ a set of
predicates $S$ if there is an injective predicate renaming~$\theta$
such that $(\eval{\P}\Dat|_{S})\theta=\eval{\P'}\Dat|_{S\theta}$ for
every dataset $\Dat$ over the signature of $\P$. The program $\P'$ is
a rewriting of $\P$ if $\P'$ is a 
rewriting of $\P$ w.r.t. the set of all predicates in
$\P$. Clearly, $\P^e$ is a rewriting of $\P$.

%% file: characterisation.tex
\section{Characterisation of Datalog Rewritability} \label{sec:characterisation}
 
%

In this section, we establish a strong correspondence between linear
disjunctive datalog and plain datalog. We show that every linear
disjunctive program can be polynomially rewritten into datalog and,
conversely, every datalog program is polynomially rewritable to a
linear disjunctive program.
Consequently, we not only can conclude that fact entailment over
linear programs has exactly the same data and combined complexity as
over plain datalog programs, but also that a disjunctive program is
datalog rewritable if and only if it is linearisable. Thus, datalog
rewritability and linearisability of disjunctive programs are
equivalent problems.


\subsubsection{From Linear Programs to Datalog}\label{sec:semi-to-datalog}

We first show that linear disjunctive programs can be polynomially
rewritten into datalog.
%
%
Let us consider the following program $\DDP_1$, which we want to rewrite
into a datalog program
$\Xi(\DDP_1)$:
  \begin{align}
    \DDP_1=\{\,&V(x)\to B(x)\lor G(x)\label{r:n2c-b-or-g}\\
    &G(y)\land E(x,y)\to B(x)\label{r:n2c-g-impl-b}\\
    &B(y)\land E(x,y)\to G(x)\,\}\label{r:n2c-b-impl-g} 
  \end{align}
 Predicates $V$ and $E$ are EDB, so their extension depends solely on
  $\Dat$. To prove facts about IDB predicates 
  $G$ and $B$ we introduce fresh binary predicates $B^G$, $B^B$, $G^B$, and
  $G^G$. Intuitively, if a fact $B^G(c,d)$ holds in $\Xi(\DDP_1) \cup
  \Dat$ then proving $B(c)$ suffices for proving $G(d)$ in
  $\DDP_1 \cup \Dat$. To ``initialise'' the extension of these fresh
  predicates we need rules 
  $\top(x) \rightarrow X^X(x,x)$ with $X \in \{G,B\}$. The key step 
  is then to ``flip'' the direction of all rules in
  $\DDP_1$ involving $G$ or $B$ by moving all IDB atoms from
  the head to the body and vice-versa while at the same time replacing
  their predicates with the relevant auxiliary predicates. Thus, 
  Rule~\eqref{r:n2c-g-impl-b} leads to the following 
  rules in $\Xi(\DDP_1)$ for each IDB predicate $X$:
\begin{align*}
&B^X(x,z)\land E(x,y)\to G^X(y,z) 
\end{align*}
These rules are natural consequences of Rule~\eqref{r:n2c-g-impl-b}
under the intended meaning of the auxiliary predicates: if we can
prove a goal $X(z)$ by proving first $B(x)$, and $E(x,y)$ holds, then
by Rule~\eqref{r:n2c-g-impl-b} we deduce that proving $G(y)$ suffices
to prove $X(z)$.  In contrast to \eqref{r:n2c-g-impl-b},
Rule~\eqref{r:n2c-b-or-g} contains no IDB body atoms. We ``flip'' this
rule as follows, with $X$ IDB:
\begin{align*}
&V(x)\land B^X(x,z)\land G^X(x,z)\to X(z) 
\end{align*}
Similarly to the previous case, this rule follows from
Rule~\eqref{r:n2c-b-or-g}: if $V(x)$ holds and we can establish
that $X(z)$ can be proved from $B(x)$ and also from $G(x)$, then
$X(z)$ must hold.
Finally, we introduce rules that allow us to derive facts
about the IDB predicates $G$ and $B$ from
facts derived about the auxiliary predicates. For example, the rule
$B(x)\land B^X(x,z)\to X(z)$ states that if $B(x)$ holds and
is sufficient to prove $X(z)$, then $X(z)$ must also hold.

\begin{definition} \label{def:xi} %
  Let $\DDP$ be a linear program and let $\Sigma$ be the set of
  IDB predicates in $\DDP\setminus\DDP_\top$.
  For each $(P,Q)\in\Sigma^2$, let $P^Q$ be a fresh predicate
  unique to $(P,Q)$ where $\arity{P^Q} = \arity{P} +
  \arity{Q}$.  Then $\Xi(\DDP)$ is the datalog program containing the
  rules given next, where $\fml$ is the conjunction of all
  EDB atoms in a rule, $\fml_\top$ is the least conjunction
  of $\top$-atoms needed to make a rule safe,
  all predicates $P_i$ are in $\Sigma$, and
  $\ve y$, $\ve z$ are disjoint vectors of distinct fresh variables:
  \begin{enumerate}
  \item a rule $\fml_\top
    \to R^R(\ve y,\ve y)$ for every $R\in\Sigma$;
  \item a rule
    $\fml_\top\land\fml\land 
    \bigwedge_{i=1}^n P_i^R(\ve s_i,\ve y)\to Q^R(\ve t,\ve y)$ for
    every rule $\fml\land Q(\ve t)\to\bigvee_{i=1}^n P_i(\ve
    s_i)\in\DDP \setminus \DDP_{\top}$ and every $R\in\Sigma$;
  \item a rule $\fml\land\bigwedge_{i=1}^n P_i^R(\ve s_i,\ve y)\to
    R(\ve y)$ for every rule $\fml\to\bigvee_{i=1}^n P_i(\ve
    s_i)\in\DDP \setminus \DDP_{\top}$ and every $R\in\Sigma$;
  \item a\;rule\;$Q(\ve z)\,{\land}\,Q^R(\ve z,\ve y)\:{\to}\,R(\ve y)$\;for\;every\;$(Q,R)\,{\in}\,\Sigma^2$.\!\!\!\!\!\!\!\!\!\!%
    \qedhere
  \end{enumerate}
\end{definition}
This transformation is quadratic and the
arity of predicates is at most doubled. 
For $\DDP_1$, we obtain the following datalog program, where each
 rule mentioning $X$ stands for one rule
 where $X=B$ and one where $X=G$:
  \begin{align}
    \Xi(\DDP_1)=\{\,
    &V(x)\land B^X(x,z)\land G^X(x,z)\to X(z) \tag{\ref{r:n2c-b-or-g}'}\label{r:n2c-rew-bx-gx-impl-x}\\
    &B^X(x,z)\land E(x,y)\to G^X(y,z) \tag{\ref{r:n2c-g-impl-b}'}\label{r:n2c-rew-bx-impl-gx}\\
    &G^X(x,z)\land E(x,y)\to B^X(y,z) \tag{\ref{r:n2c-b-impl-g}'}\label{r:n2c-rew-gx-impl-bx}\\
    &\top(x) \to X^X(x,x)\label{r:n2c-rew-xx}\\
    &B(x)\land B^X(x,z)\to X(z)\label{r:n2c-rew-b-bx-impl-x}\\
    &G(x)\land G^X(x,z)\to X(z)\,\}\label{r:n2c-rew-g-gx-impl-x}
  \end{align}
Correctness of $\Xi$ is established
by the following theorem.
\begin{restatable}{theorem}{rewcorrectlinear} \label{thm:rew-correct}
  If\/ $\DDP$ is linear, then $\Xi(\DDP)$ is a polynomial datalog
  rewriting of\/ $\DDP$.
\end{restatable}

\begin{figure*}[t]
  \hskip4mm
    \begin{minipage}[t]{8cm}\vspace{0pt}
      \begin{small}
        \begin{tikzpicture}[parent anchor=south,sibling distance=25mm,level distance=7.1mm]
          \tikzstyle{level 4}=[level distance=9mm]
          \node (ba) {$B(a)$}
          child {node {$G(b)\lor B(a)$}
            child {node {$B(c)\lor B(a)$}
              child {node {$B(c)\lor G(c)$}
                child {node {$\quad~~~~V(c)\in\Dat_1\!\!\!$}
                  edge from parent node[right=-.2mm] {\eqref{r:n2c-b-or-g}}}}
              child {node {$E(a,c)\in\Dat_1\!$}
              }}
            child {node {$E(b,c)\in\Dat_1\!$}}}
          child {node {$E(a,b)\in\Dat_1\!$}};
          \node[below of=ba,node distance=5.3mm] {\eqref{r:n2c-g-impl-b}};
          \node[below of=ba-1,node distance=5.3mm] {\eqref{r:n2c-b-impl-g}};
          \node[below of=ba-1-1,node distance=5.3mm] {\eqref{r:n2c-g-impl-b}};
          \node[left of=ba,node distance=4cm] {(a)};
        \end{tikzpicture}
      \end{small}
    \end{minipage}\hskip2mm%
    \begin{minipage}[t]{9cm}\vspace{0pt}
      \begin{small}
        \begin{tikzpicture}[parent anchor=south,level distance=10mm]
          \tikzstyle{level 1}=[sibling distance=21mm]
          \tikzstyle{level 2}=[sibling distance=17mm]
          \tikzstyle{level 3}=[sibling distance=17mm]
          \node (ba) {$B(a)$}
          child {node {$B^B(c,a)$}
            child {node {$G^B(b,a)$}
              child {node {$B^B(a,a)$}
              }
              child {node {$\quad~E(a,b)\in\Dat_1\!\!$}}}
            child {node {$\quad~E(b,c)\in\Dat_1\!\!$}}}
          child[grow=-120,level distance=14mm] {node {$\quad~~~~~~~V(c)\in\Dat_1\!\!$}
            edge from parent node[right] {\eqref{r:n2c-rew-bx-gx-impl-x}}}
          child {node {$G^B(c,a)$}
            child {node {$B^B(a,a)$}
            }
            child {node {$\quad~E(a,c)\in\Dat_1\!\!$}
            }};
          \node[below of=ba-1,node distance=6.5mm] {\eqref{r:n2c-rew-gx-impl-bx}};
          \node[below of=ba-1-1,node distance=6.5mm] {\eqref{r:n2c-rew-bx-impl-gx}};
          \node[below of=ba-3,node distance=6.5mm] {\eqref{r:n2c-rew-bx-impl-gx}};
          \node[left of=ba,node distance=4cm] {(b)};
        \end{tikzpicture}
      \end{small}
    \end{minipage}
    \caption{(a) derivation of $B(a)$ from $\DDP_1\cup\Dat_1$;\, (b) derivation of $B(a)$ from $\Xi(\DDP_1)\cup\Dat_1$}
  \label{fig:derivation-p1}
\end{figure*}

Thus, fact entailment over linear programs is no harder than in
datalog: \textsc{PTime} w.r.t. data and \EXPTIME in combined
complexity.
Formally, Theorem \ref{thm:rew-correct} is shown by induction on
hyperresolution derivations of facts entailed by the rules in $\DDP$
from a given dataset $\Dat$ (
see Appendix).
We next sketch the intuitions
%
%
on 
$\DDP_1$ and $\Dat_1=\set{V(a),$ $V(b),$ $V(c),$ $E(a,b),$
  $E(b,c),$ $E(a,c)}$. 

Figure~\ref{fig:derivation-p1}, Part (a) shows a linear
(hyperresolution) derivation $\rho_1$ of $B(a)$ from
$\DDP_1\cup\Dat_1$ while Part (b) shows a derivation $\rho_2$ of the
same fact from $\Xi(\DDP_1)\cup\Dat_1$.
We represent derivations 
as trees whose nodes are labeled with disjunctions of facts and where
every inner node is derived from its children using a rule of the
program (initialisation rules in $\rho_2$ 
are omitted for brevity).  We first show that if $B(a)$ is provable in
$\DDP_1 \cup \Dat_1$, then it is entailed by $\Xi(\DDP_1) \cup \Dat_1$.
From the premise, a linear derivation such as $\rho_1$ exists. The
crux of the proof is to show that each disjunction of facts in
$\rho_1$ corresponds to a set of facts over the auxiliary predicates
entailed by $\Xi(\DDP_1) \cup \Dat_1$. Furthermore, these facts must be
of the form $X^B(u,a)$, where $B(a)$ is the goal, $u$ is a constant,
and $X \in \{B,G\}$.  For example, $B(c) \vee G(c)$ in $\rho_1$
corresponds to facts $B^B(c,a)$ and $G^B(c,a)$, which are provable
from $\Xi(\DDP_1) \cup \Dat_1$, as witnessed by $\rho_2$. Since $\rho_1$
is linear, it has a unique rule application that has only EDB atoms as
premises, i.e., the application of \eqref{r:n2c-b-or-g}, which
generates $B(c) \vee G(c)$. Since $B^B(c,a)$ and $G^B(c,a)$ are
provable from $\Xi(\DDP_1) \cup \Dat_1$, we can apply
\eqref{r:n2c-rew-bx-gx-impl-x} to derive $B(a)$.

Finally, we show the converse: if $B(a)$ is provable from $\Xi(\DDP_1)
\cup \Dat_1$ then it follows from $\DDP_1 \cup \Dat_1$. For this, we
take a derivation such as $\rho_2$, and show that each fact in
$\rho_2$ about an auxiliary predicate carries the intended meaning,
e.g., for $G^B(b,a)$ we must have $\DDP_1 \cup \Dat_1 \models G(b)
\rightarrow B(a)$.

\subsubsection{From Datalog to Linear Programs}\label{sec:datalog-to-linearDD}



The transformation from datalog to linear disjunctive datalog is based
on the same ideas, but it
is simpler in that we no longer distinguish between EDB and
IDB atoms: a 
rule in $\P$ is now ``flipped'' by moving \emph{all} its atoms from
the head to the body and vice-versa.
Moreover, we make 
 use of the IDB expansion $\P^e$ of
$\P$ to ensure linearity of the resulting disjunctive program.
%
\begin{definition} \label{def:psi}
  Let $\P$ be a datalog program. For each pair $(P,Q)$ of
  predicates in $\P$, let $P^Q$ be a fresh predicate unique
  to $(P,Q)$ where $\arity{P^Q} = \arity{P} + \arity{Q}$.
  Furthermore, let $\P^e$ be the IDB expansion of $\P$. Then, $\Psi(\P)$ is the linear disjunctive program containing, for each
  IDB predicate $R$ in $\P^e$
  the rules given next, where $\fml_\top$ is the least
  conjunction of $\top$-atoms making a rule safe and $\ve y=y_1\dots
  y_{\arity{R}}$ is a vector of distinct fresh variables:
  \begin{enumerate}
  \item a rule $\fml_\top\land Q^R(\ve t,\ve y)\to\bigvee_{i=1}^n
    P_i^R(\ve s_i,\ve y)$ for every rule $\bigwedge_{i=1}^n P_i(\ve
    s_i)\to Q(\ve t)\in\P^e\setminus \P^e_{\top}$, where $Q(\ve t)\ne\bot$ and $\bigvee_{i=1}^n P_i^R(\ve
    s_i,\ve y)$ is interpreted as $\bot$ if $n=0$;
  \item a rule $\fml_\top 
    \to\bigvee_{i=1}^n P_i^R(\ve s_i,\ve y)$ for every 
    $\bigwedge_{i=1}^n P_i(\ve s_i)\to\bot\in\P^e\setminus \P^e_{\top}$;
  \item a rule $\fml_\top 
    \to R^R(\ve y,\ve y)$;
  \item a rule $Q(\ve z)\land Q^R(\ve z,\ve y)\to R(\ve y)$ for every
    EDB predicate $Q$ in $\P^e$, where $\ve z$ is a vector of distinct
    fresh variables.\qedhere
  \end{enumerate}

\end{definition}
Again, the transformation is quadratic
and the arity of predicates is at most doubled.
\begin{example}
Consider $\P_2$, which encodes 
path system accessibility (a canonical
\textsc{PTime}-complete~problem):
\begin{align}
  \P_2=\{\,
  &R(x,y,z)\land A(y)\land A(z)\to A(x)\,\}\label{eq:ps-step}
\end{align}
Linear datalog is
\textsc{NLogSpace}, and cannot capture
$\P_2$. However, we can rewrite
$\P_2$ into linear disjunctive datalog:
\begin{align}
  \Psi(\P_2)=\{\,
  &\top(y)\land\top(z)\land {A'}^{A'}(x,u)\tag{\ref{eq:ps-step}'}\label{eq:ps-rew-step}\\
  &\!\to R^{A'}(x,y,z,u)\lor {A'}^{A'}(y,u)\lor {A'}^{A'}(z,u)\nonumber\\
  &{A'}^{A'}(x,y)\to A^{A'}(x,y)\label{eq:ps-rew-init}\\
  &\top(x)\to {A'}^{A'}(x,x)\label{eq:ps-rew-tt}\\
  &A(x)\land A^{A'}(x,y)\to {A'}(y)\label{eq:ps-rew-a-at-impl-t}\\
  &R(x,y,z)\land R^{A'}(x,y,z,u)\to {A'}(u)\,\}\label{eq:ps-rew-r-rt-impl-t}
\end{align}

Rule~\eqref{eq:ps-step}
yields Rule~\eqref{eq:ps-rew-step} in $\Psi(\P_2)$. 
Rule~\eqref{eq:ps-rew-init} is obtained
from $A(x)\to A'(x)\in\P^e$.
To see why we need the IDB expansion $\P^e$, suppose
we replaced $\P^e$ by $\P$ in Definition~\ref{def:psi}.
Rule~\eqref{eq:ps-rew-init} would not be produced and $A'$ would be
replaced by $A$ elsewhere. Then the rule $A(x)\land A^{A}(x,y)\to {A}(y)$ would
not be linear since both $A$ and $A^A$ would be IDB.
\end{example}
Correctness of $\Psi$ is established by the following theorem.

\begin{restatable}{theorem}{invrewcorrect} \label{thm:inv-rev-correct} %
  If\/ $\P$ is datalog, then $\Psi(\P)$ is a polynomial rewriting of\/
  $\P$ into a linear disjunctive program.
\end{restatable}

%
%
From Theorems~\ref{thm:rew-correct} and \ref{thm:inv-rev-correct} we obtain
the following results.
%
\begin{corollary} 
  A disjunctive program\/ $\DDP$ is datalog rewritable iff\/
  it is rewritable into a linear disjunctive program.
\end{corollary}
%
%
\begin{corollary}
  Checking  $\DDP \cup \Dat \models \alpha$ for\/
  $\DDP$ a linear program, $\Dat$ a dataset and
  $\alpha$ a fact is \textsc{PTime}-complete w.r.t. data complexity and
  \EXPTIME-complete w.r.t. combined complexity.
\end{corollary}


%% file: weaklinearity.tex
\section{Weakly Linear Disjunctive Datalog} \label{sec:weak}

In this section, we introduce $\semi$ programs: a new
class of disjunctive datalog programs that extends both datalog
and linear disjunctive datalog. 
 The main idea is simple: instead of requiring 
the body of 
each rule to contain at most one occurrence of an IDB predicate,
we require at most one occurrence of a \emph{disjunctive} predicate---a predicate
whose extension for some dataset
could depend on the application of a disjunctive rule. 
This intuition is formalised as given next.

\begin{definition}
  The \emph{dependency graph} $G_\DDP=(V,E,\mu)$ of a program $\DDP$ is the
  smallest edge-labeled digraph such that:
  \begin{enumerate}
  \item $V$ contains every predicate occurring in $\DDP$;
  \item $r\in\mu(P,Q)$ whenever $P,Q \in V$,  $r\in\DDP\setminus\DDP_\top$, $P$ 
  occurs in the body of $r$, and $Q$ occurs in the head of $r$; and
  \item $(P,Q)\in E$ whenever $\mu(P,Q)$ is nonempty.
  \end{enumerate}
  A predicate $Q$ \emph{depends on a rule $r\in\DDP$} if $G_\DDP$ has
  a path that ends in $Q$ and involves an $r$-labeled
  edge. Predicate $Q$ is \emph{datalog} if it
  only depends on datalog rules; 
  otherwise, $Q$ is \emph{disjunctive}.
  Program $\DDP$ is \emph{$\semi$} (WL for short) if every rule in $\DDP$
  has at most one occurrence of a disjunctive predicate in the body.
\end{definition}
Checking whether $\DDP$ is WL is clearly feasible in
polynomial time. If $\DDP$ is datalog, then all its
predicates are datalog and $\DDP$ is WL. Furthermore, every
disjunctive predicate is IDB and hence every linear program is also
WL. There are, however, WL programs that are neither datalog
nor linear. Consider $\DDP_3$, which extends
$\DDP_1$ with the following rule:
\begin{align}
  &E(y,x)\to E(x,y)\label{r:n2c-e-impl-e}
\end{align}
Since $E$ is IDB in $\DDP_3$, Rules \eqref{r:n2c-g-impl-b} and
\eqref{r:n2c-b-impl-g} have two IDB atoms. Thus, $\DDP_3$ is not
linear.
The 
 graph $G_{\DDP_3}$ 
looks as follows.
\begin{center}
  \begin{small}
    \begin{tikzpicture}[>=stealth,->]
      \matrix[matrix of math nodes,
      row sep={.65cm,between origins},
      column sep={1.5cm,between origins}]
      {
        && |(B)| B & &\\
        \top&|(V)| V & & |(E)| E &\phantom{aaaaaaa}\bot\\
        && |(G)| G & &\\
      };
      \draw (V) -- node [above left=-2pt] {(\ref{r:n2c-b-or-g})} (B);
      \draw (V) -- node [below left=-2.2pt] {(\ref{r:n2c-b-or-g})} (G);
      \draw (E) -- node [above right=-2pt] {(\ref{r:n2c-g-impl-b})} (B);
      \draw (E) -- node [below right=-2.2pt] {(\ref{r:n2c-b-impl-g})} (G) ;
      \draw (B) to [bend left] node [right=-.5pt] {(\ref{r:n2c-b-impl-g})} (G);
      \draw (G) to [bend left] node [left=-.5pt] {(\ref{r:n2c-g-impl-b})} (B);
      \draw (E) to [loop right] node [right=-.5pt] {(\ref{r:n2c-e-impl-e})} (E);
    \end{tikzpicture}
  \end{small}
\end{center}
\vskip-1.1mm 
Predicate $V$ is EDB and hence does not depend on any rule. Predicates
$B$ and $G$ depend on Rule~\eqref{r:n2c-b-or-g} and hence are
disjunctive. Finally, predicate $E$ depends only on
Rule~\eqref{r:n2c-e-impl-e} and hence it is a datalog predicate.
Thus, $\DDP_3$ is WL.
\begin{definition}
For $\DDP$ WL,
let $\Xi'(\DDP)$ be defined as $\Xi(\DDP)$
in Definition~\ref{def:xi} but where:
\emph{(i)}  $\Sigma$ is the set of all disjunctive predicates in
  $\DDP\setminus\DDP_\top$; \emph{(ii)} 
  $\fml$ denotes the conjunction of all datalog atoms in a rule;
and \emph{(iii)} in addition to rules  (1)--(4), $\Xi'(\DDP)$ contains
  every rule in $\DDP$ with no disjunctive predicates.\qedhere
\end{definition}

By adapting the proof of Theorem~\ref{thm:rew-correct} we obtain:

\begin{restatable}{theorem}{rewcorrect} \label{thm:rew-correct-semi} %
  If\/ $\DDP$ is WL, then $\Xi'(\DDP)$ is a polynomial datalog rewriting
  of\/ $\DDP$.
\end{restatable}
Thus, fact entailment over WL programs has the same data and
combined complexity as for datalog. 
Furthermore, $\Xi'(\DDP)$ is a rewriting of
$\DDP$ and hence it preserves the extension of all predicates.
If, however, we want to query a specific predicate $Q$, 
we can compute a smaller program, which is linear in the size
of $\DDP$ and preserves the extension of $Q$. Indeed, if $Q$ is datalog, 
each proof in $\DDP$ of a fact about $Q$
involves only datalog rules, and if $Q$ is disjunctive, each
such proof involves only auxiliary predicates $X^Q$. Thus, in $\Xi'$ we can dispense with all rules
involving auxiliary predicates $X^R$ for $R\neq Q$. In
particular, if $Q$ is datalog, the rewriting contains no auxiliary
predicates. 
%
%
\begin{restatable}{theorem}{rewcorrectsinglequery} \label{sec:rew-correct-single-query}
  Let\/ $\DDP$ be WL, $S$ a set of 
  predicates in $\DDP$, and $\DDP'$ obtained from $\Xi'(\DDP)$
  by removing all rules with a predicate $X^R$ for $R \not\in
  S$. Then $\DDP'$ is a rewriting of\/ $\DDP$ w.r.t. $S$.
\end{restatable}


%% file: procedure.tex
\section{Rewriting Programs via Unfolding} \label{sec:procedure}

Although  WL programs 
can be rewritten into datalog,
not all datalog rewritable  programs
are WL. Let $\DDP_4$ be as follows:
%
\begin{align}
  \DDP_4=\{\,&A(x)\land B(x)\to C(x)\lor D(x)\label{eq:nsl-ab-impl-cd} \\
  &E(x) \to A(x)\lor F(x)\label{eq:nsl-e-impl-af} \\
  &C(x)\land R(x,y)\to B(y)\,\}\label{eq:nsl-cr-impl-b}
\end{align}
Program $\DDP_4$ is not WL since both body atoms in 
\eqref{eq:nsl-ab-impl-cd} are disjunctive. 
However, $\DDP_4$
is datalog rewritable.

We now present a rewriting procedure that combines our
results in Section~\ref{sec:weak}
with the work of \mbox{\citeA{Gergatsoulis97}} on program
transformation for disjunctive logic programs.
Our procedure iteratively eliminates non-WL rules by
``unfolding'' the culprit atoms w.r.t.\ the other 
rules in the program. It stops when the 
program becomes WL, and outputs a datalog program as in
Section~\ref{sec:weak}.  The procedure is sound: if it succeeds, 
the output is a datalog rewriting. 
It is, however, both incomplete 
(linearisability cannot be semi-decided just by unfolding) and
non-terminating.  Nevertheless, our experiments 
suggest that unfolding can be effective in
practice since some programs obtained from realistic ontologies
can be rewritten into datalog after a few unfolding steps.


\subsubsection{Unfolding}


We start by recapitulating 
\cite{Gergatsoulis97}. Given a disjunctive program~$\DDP$, a rule~$r$
in~$\DDP$, and a body atom~$\alpha$ of~$r$, Gergatsoulis defines the
\emph{unfolding} of~$r$ at~$\alpha$ in~$\DDP$ as a transformation
of~$\P$ that replaces~$r$ with a set of resolvents of~$r$ with rules
in~$\P$ at~$\alpha$ (see Appendix).
We denote the resulting program by $\mathsf{Unfold}(\DDP,r,\alpha)$.
Unfolding preserves all entailed disjunctions $\fml$ of facts: $\DDP\models\fml$ iff
$\mathsf{Unfold}(\DDP,r,\alpha)\models\fml$ for all~$\DDP$,~$r$,~$\alpha$,~and~$\fml$. 
However, to ensure that 
unfolding produces a rewriting we need a stronger
correctness result that is dataset independent.


\begin{restatable}{theorem}{unfoldingcorrect}
  \label{thm:unfolding-correct} %
  Let $\DDP_0$ be a disjunctive program and let $\DDP$ be a rewriting
  of\/ $\DDP_0$ such that no IDB predicate in\/ $\DDP$ occurs
  in~$\DDP_0$. Let\/ $r$ be a rule in $\DDP$ and $\alpha$ be an IDB body
  atom of\/~$r$. Then $\mathsf{Unfold}(\DDP,r,\alpha)$ is a rewriting
  of\/~$\DDP_0$. Moreover, no IDB predicate in
  $\mathsf{Unfold}(\DDP,r,\alpha)$ occurs in\/~$\DDP_0$.
\end{restatable}




\subsubsection{The Rewriting Procedure} \label{sec:procedure-proper}

\begin{procedure}[t]
\caption{$\mathsf{Rewrite}$}\label{alg:rewrite}
\begin{footnotesize}
    \textbf{Input:}
    $\DDP$: a disjunctive program \\
    \textbf{Output:}
    a datalog rewriting of $\DDP$
    \begin{algorithmic}[1]
      \State $\DDP':=\DDP^e$
      \While{$\DDP'$ not WL}
      \State select $r\,{\in}\,\DDP'$ with more than one disjunctive body atom

      \State select a disjunctive body atom $\alpha\in r$ 

      \State $\DDP':= \mathsf{Unfold}(\DDP',r,\alpha)$
      \EndWhile
      \State \Return $\Xi'(\P')$
    \end{algorithmic}
\end{footnotesize}
\end{procedure}

Procedure~\ref{alg:rewrite} attempts to eliminate rules with several
disjunctive body atoms by unfolding one such atom.  Note that to
satisfy the premise of Theorem~\ref{thm:unfolding-correct},
unfolding is applied to $\DDP^e$ rather than $\DDP$.
Correctness of Procedure~\ref{alg:rewrite} is established by the
following theorem.

\begin{restatable}{theorem}{procedure-correct}\label{th:rewrite-procedure}
  Let $\DDP$ be a disjunctive program. If $\mathsf{Rewrite}$
  terminates on $\DDP$ with output\/ $\P'$, then\/
  $\P'$ is a rewriting of\/ $\DDP$.
\end{restatable}
%
%
%
%
$\mathsf{Rewrite}$ first transforms our example program $\DDP_4$
 to
\begin{align}
  \DDP'_4=\{\,&A'(x)\land B'(x)\to C'(x)\lor D'(x)\label{eq:nslp-ab-impl-cd} \\
  &E(x) \to A'(x)\lor F'(x)\label{eq:nslp-e-impl-af} \\
  &C'(x)\land R(x,y)\to B'(y)\,\}\cup\DDP_{\mathrm{aux}}\label{eq:nslp-cr-impl-b}
\end{align}
where $\DDP_{\mathrm{aux}}=\mset{P(x)\to P'(x)}{P\in\set{A,B,C,D,F}}$
and $A',B',C',D',F'$ are fresh. Rule \eqref{eq:nslp-ab-impl-cd}
is not WL in $\DDP'_4$,
and  needs to be unfolded.
We choose to unfold~\eqref{eq:nslp-ab-impl-cd} on $A'(x)$.
Thus, in Step~5, Rule~\eqref{eq:nslp-ab-impl-cd} is replaced by the
rules
\begin{align}
  &A(x)\land B'(x)\to C'(x)\lor D'(x)\label{eq:nslp-rew-ab-impl-cd} \\
  &E(x)\land B'(x)\to C'(x)\lor D'(x)\lor F'(x)\label{eq:nslp-rew-eb-impl-cdf}
\end{align}
The resulting
$\DDP''_4$ is WL, and $\mathsf{Rewrite}$ 
returns $\Xi'(\DDP''_4)$.


%% file: applicationDLs.tex
\section{Application to OWL Ontologies} \label{sec:DLs}



The RL profile is a 
fragment of OWL 2 for which reasoning
is tractable and practically realisable
by means of rule-based technologies. 
RL is also a fragment of datalog: 
each RL ontology can be normalised to a datalog program. 

\begin{table}[t]
\small
\begin{displaymath}
\begin{array}{@{}l@{\quad}r@{\;}l@{\qquad}l@{}}
    1.  & A             & \sqsubseteq \atmostq{1}{R}{B} & A(z) \wedge R(z,x_1) \wedge B(x_1)\\
    &&&~~\wedge R(z,x_2) \wedge B(x_2) \rightarrow x_1 \equality x_2 \\
    2.  & A \sqcap B    & \sqsubseteq C                 & A(x) \wedge B(x) \rightarrow C(x) \\
    3.  & \exists R.A   & \sqsubseteq B                 & R(x,y) \wedge A(y) \rightarrow B(x) \\
    4.  & R             & \sqsubseteq S                 & R(x_1,x_2) \rightarrow S(x_1,x_2) \\
    5.  & R \circ S     & \sqsubseteq T                 & R(x_1,z) \wedge S(z,x_2) \rightarrow T(x_1,x_2) \\
    6.  & A             & \sqsubseteq \mathsf{Self}(R)  & A(x) \rightarrow R(x,x) \\
    7.  & \mathsf{Self}(R) & \sqsubseteq A              &  R(x,x) \rightarrow A(x) \\ 
    8. & R             & \sqsubseteq S^-               & R(x,y) \rightarrow S(y,x) \\
    9.  & A             & \sqsubseteq \{ a \}           & A(x) \rightarrow x \equality a \\ 
    10. & \{a\}        & \sqsubseteq A            & A(a) \\
    \hline
    11. & A             & \sqsubseteq B \sqcup C  &  A(x) \rightarrow B(x) \vee C(x)  \\
\end{array}
\end{displaymath}
\caption{Normalised $\textup{RL}^{(\sqcup)}$ 
  axioms, with $A,B$ atomic
or $\top$,  $C$   atomic or $\bot$, 
$R,S,T$ atomic roles, $a$ an individual.
}
\label{tab:RL}
\end{table}

We next show how to extend RL with disjunctions while retaining
tractability of consistency checking and fact entailment in
\emph{combined complexity}.
We first recapitulate the kinds of normalised axioms that can occur in
an RL ontology.  We assume familiarity with Description Logic (DL)
notation.
  
A (normalised) \emph{RL ontology} is a finite set of DL axioms of the form 1-10
in Table~\ref{tab:RL}.
%
The table also provides the translation of DL axioms into rules.
We define $\RLor$ as the extension of RL
with axioms 
capturing disjunctive knowledge.
\begin{definition}
  An \emph{$\RLor$ ontology} is a finite set of DL axioms of the form
  1-11 in Table~\ref{tab:RL}.
\end{definition}
%

Fact entailment in $\RLor$ is 
\coNP-hard since 
$\RLor$ can encode non-3-colourability. 
Membership in $\coNP$ holds since 
rules have bounded number of variables, and hence programs can be
grounded 
in polynomial time (see Appendix).
%
%
Tractability 
can be regained if we restrict ourselves to $\RLor$ ontologies 
corresponding to WL programs. 
WL programs~$\DDP$ obtained from $\RLor$ ontologies have bounded
number of variables, and thus variables in
$\Xi'(\DDP)$ are also bounded. 
%
\begin{restatable}{theorem}{semiRLorComplexity} \label{th:semiRLor}
Checking $\Ont \cup \Dat \models \alpha$, for\/
$\Ont$ an $\RLor$ ontology that corresponds to a 
WL program,
is \textsc{PTime}-complete.
\end{restatable}
Thus, fact entailment in $\RLor$ is no
harder than in RL, and one can use
scalable engines
such as RDFox.  
Our experiments indicate that many
ontologies captured by $\RLor$ are either
WL or can be made WL via unfolding, which makes
data reasoning over such ontologies feasible.%
\footnote{For CQ answering, our language becomes \coNP-hard
w.r.t.\ data,
whereas RL is tractable. This 
follows from \cite{Lutz:2012ug} already for a single axiom of type 11.}

\subsubsection{Dealing with Expressive Ontology Languages}
 
\citeA{hms07reasoning} developed an algorithm for transforming $\SHIQ$
ontologies into an equivalent
disjunctive datalog program. \citeA{grauIJCAI13} combined this algorithm with a
knowledge compilation procedure (called
$\mathsf{Compile\text{-}Horn}$) obtaining a sound but incomplete and
non-terminating datalog rewriting procedure for $\SHIQ$. Our
procedure $\mathsf{Rewrite}$ provides an alternative
to $\mathsf{Compile\text{-}Horn}$ for $\SHIQ$. 
The classes of ontologies rewritable by the two
procedures can be shown incomparable (e.g., 
$\mathsf{Compile\text{-}Horn}$ may not terminate on WL programs).

%% file: related.tex
\section{Related Work}

Complexity of disjunctive datalog with negation as failure
has been extensively studied 
\cite{BenEliyahuZoharyPalopoli97,DBLP:journals/tods/EiterGM97}.
The class of \emph{head-cycle} free programs was
studied in
\citeA{BenEliyahuZoharyPalopoli97,BenEliyahuZoharyEtAl00}, where it was
shown that certain reasoning problems are tractable
for such programs 
(fact entailment, however, remains intractable w.r.t.\ data).

\citeA{GottlobMMP12} investigated complexity of 
disjunctive TGDs and showed tractability (w.r.t. data complexity) of
fact entailment for a class of linear disjunctive TGDs.
%
Such rules allow for existential quantifiers in the head, but 
require single-atom bodies;
thus, they
are incomparable to WL rules.
\citeA{Artale09thedl-lite} showed tractability of fact entailment
w.r.t.\ data for $\text{DL-Lite}_{\mathsf{bool}}$ logics. This result is related to \cite{GottlobMMP12} since certain
$\text{DL-Lite}_{\mathsf{bool}}$ logics can be represented as linear
disjunctive TGDs.  Finally, combined complexity of CQ answering for
disjunctive TGDs was studied by \citeA{DBLP:conf/ijcai/BourhisMP13}.

\citeA{Lutz:2012ug} investigated non-uniform data complexity of CQ answering
w.r.t.\ extensions of \ALC{}, and  
related CQ answering to constraint satisfaction
problems. This connection was  explored by
\citeA{CarstenPODS2013}, who showed
\textsc{NExp\-Time}-completeness of first-order
 and datalog rewritability of instance queries
for $\mathcal{SHI}$. 

The procedure in~\cite{grauIJCAI13}, mentioned in
Section~\ref{sec:DLs},
is used by \mbox{\citeA{KaminskiG13}} to show first-order/datalog
rewritability of two fragments of $\ELU$. Notably, both
fragments yield linear
programs.
Finally, our unfolding-based rewriting procedure is
motivated by the work of~\citeA{AfratiEtAl03} on linearisation of
plain datalog programs by means of program transformation
techniques~\cite{TamakiS84,ProiettiP93,Gergatsoulis97}.


%% file: evaluation.tex
\section{Evaluation}\label{sec:evaluation}

\noindent
\textbf{Rewritability Experiments. }  
We have evaluated whether realistic
ontologies can be rewritten to WL (and hence to datalog)
programs. We analysed 118 non-Horn ontologies from 
BioPortal,
the Prot\'eg\'e 
library,
and the corpus in \cite{GardinerTH06}. To transform ontologies into
disjunctive datalog we used KAON2
\mbox{\cite{Motik06phd}}.\footnote{We doctored
the ontologies to remove constructs outside $\mathcal{SHIQ}$, and hence not supported by KAON2. 
The modified ontologies can be found
  on 
http://csu6325.cs.ox.ac.uk/WeakLinearity/%
}
KAON2 succeeded to compute disjunctive
programs for 103 ontologies.
On these, 
$\mathsf{Rewrite}$ succeeded in 35 cases: 8 programs were already
datalog after CNF
normalisation, 12 were linear, 12 were WL, and 3 required
unfolding. $\mathsf{Rewrite}$ was limited to 
1,000
unfolding steps, but all successful cases required at most
11 steps. On average, 73\% of the predicates in ontologies were datalog, and so
could be queried using a datalog engine (even if the
disjunctive program could not be rewritten).  We identified 15 $\RLor$
ontologies and obtained WL programs for 13 of them.  For
comparison, we implemented the procedure
$\mathsf{Compile\text{-}Horn}$ in \cite{grauIJCAI13}, which succeeded
on 18 ontologies, only one of which
could not be rewritten by our approach.

\smallskip
\noindent 
\textbf{Query Answering. } 
We tested scalability of instance query answering
using datalog programs obtained by our approach.  For this, we
used UOBM and DBpedia, which come with large
datasets. UOBM \cite{DBLP:conf/esws/MaYQXPL06} is a standard benchmark
for which synthetic data is available~\cite{DBLP:conf/www/ZhouGHWB13}. 
We denote the dataset for $k$ universities by~U$k$. We 
considered the $\RLor$ subset of UOBM (which is
rewritable using $\mathsf{Rewrite}$ but not using
$\mathsf{Compile\text{-}Horn}$), and generated
datasets U01, U04, U07, U10.
%
DBpedia\footnote{http://dbpedia.org/About} is a realistic ontology
with a large dataset from Wikipedia. Since DBpedia is Horn, we
extended it with reasonable disjunctive
axioms. 
\begin{table}{\small
	\begin{tabular}{l|r@{~~~~}r@{~~~\,}r|r@{~~\,}r@{\quad}r|r@{~~}r@{\quad\,}r}
				& \multicolumn{3}{c}{Our approach}	& \multicolumn{3}{|c}{HermiT}	& \multicolumn{3}{|c}{Pellet}	\\
				& dlog & disj & err		 	& dlog  & disj  & err			& dlog		& disj & err	\\\hline
		U01		& $<$1s &  8s &              	&   6s   & 107s   &				& 146s   	& 172s  & 		\\
		U04		& $<$1s & 55s &               	&  50s   &  50s   & 2			& ---     	& ---  & ---	\\
		U07		& $<$1s & 62s & 3             	& 107s   & 122s   & 2			& ---     	& ---  & ---	\\
		U10		& $<$1s & 66s & 5             	& 176s   & 182s   & 2			& ---     	& ---  & --- 
	\end{tabular}
}\caption{Average query answering times}\label{tab:queryAnswering}
\end{table}
We used RDFox as a datalog engine. Performance 
was measured against HermiT
\cite{msh09hypertableau} 
and Pellet
\cite{SirinEtAl07}. 
We used a server with two Intel Xeon E5-2643
processors and 128GB RAM. Timeouts
were $10$min for one query and $30$min for all queries;
a limit of 100GB was allocated to 
each task.  We ran RDFox on 16 threads.  Systems were compared on
individual queries, and on precomputing answers to all
queries. All systems succeeded to answer all queries for U01: HermiT
required 890s, Pellet 505s, and we 
52s. Table~\ref{tab:queryAnswering} depicts average 
times for datalog and disjunctive predicates, and number of
queries on which a system failed.\footnote{Average
times do not reflect queries on which a system failed.}
Pellet only succeeded to answer queries
on U01. 
HermiT's performance was similar for datalog and disjunctive
predicates. In 
our case, queries over the 130 datalog predicates 
in UOBM (88\% of all
predicates) were answered instantaneously~($<$1s);
queries over disjunctive predicates were harder, since
the rewritings expanded the dataset
quadratically in some cases. 
Finally, due to its size,
DBpedia's dataset 
cannot even be loaded by  HermiT or Pellet.
Using 
RDFox, our rewriting precomputed the answers for all
DBpedia predicates in 48s.

%% file: discussion.tex
\section{Conclusion}

We have proposed a characterisation of
datalog rewritability for disjunctive datalog 
programs, as well as tractable fragments of disjunctive datalog. 
Our techniques can be applied to rewrite OWL ontologies into datalog,
which enables the use of scalable datalog engines for data
reasoning. Furthermore, our approach is not ``all or nothing'': even
if an ontology cannot be rewritten, we can still answer queries over
most (i.e., datalog) predicates using a datalog reasoner.



%% file: proofs-characterisation.tex
\section{Proofs for Section~\ref{sec:characterisation}}


\begin{definition}
  Let $r=\bigwedge_{i=1}^n \beta_i\to\fml$ be a rule and, for each
  $1\le i\le n$, let $\fmm_i$ be a disjunction of facts
  $\fmm_i=\fmn_i\lor \alpha_i$ with $\alpha_i$ a single fact. Let
  $\sigma$ be an MGU of each $\beta_i,\alpha_i$. Then the following
  disjunction of facts $\fml'$ is a \emph{hyperresolvent} of $r$ and
  $\fmm_1,\dots,\fmm_n$:
  $\fml'=\fml\sigma\lor\fmn_1\lor\dots\lor\fmn_n$.\footnote{We view
    disjunctions as sets of formulae.}
\end{definition}

\begin{definition}
  Let $\DDP$ be a program, let $\Dat$ be a dataset, and let $\fml$ be
  a disjunction of facts. A (hyperresolution) \emph{derivation} of
  $\fml$ from $\DDP\cup\Dat$ is a pair $\rho=(T,\lambda)$ where $T$ is
  a tree, $\lambda$ a labeling function mapping each node in $T$ to a
  disjunction of facts, and the following properties hold for each
  $v\in T$:
  \begin{enumerate}
  \item $\lambda(v) = \fml$ if $v$ is the root;
  \item $\lambda(v) \in\DDP\cup\Dat$ if $v$ is a leaf; and
  \item if\/ $v$ has children $w_1, \ldots, w_n$, then $\lambda(v)$ is
    a hyperresolvent of a rule $r \in \DDP$ and $\lambda(w_1), \ldots,
    \lambda(w_n)$.\qedhere
  \end{enumerate}
\end{definition}

%
%
We write $\DDP\cup\Dat\vdash\fml$ to denote that $\fml$ has a
derivation from $\DDP\cup\Dat$. Hyperresolution is sound and complete
in the following sense: If $\DDP\cup\Dat$ is unsatisfiable, then
$\DDP\cup\Dat\vdash\bot$, and otherwise $\DDP\cup\Dat\vdash \alpha$
iff $\alpha \in \eval\DDP\Dat$.\footnote{This implies that
  $\DDP\cup\Dat\vdash\bot$ iff $\eval\DDP\Dat=\set{\bot}$.}

\begin{definition}
  Let $\DDP$ be a (disjunctive) program and $\Dat$ be a dataset. A
  $\top$\emph{-stub} is a one-step derivation of a fact $\top(a)$ (for
  some~$a$) from $\Dat$ using a rule in $\P_\top$. A
  derivation $\rho=(T,\lambda)$ from $\P\cup\Dat$ is \emph{normal} if
  every node whose label involves $\top$ is the root of a $\top$-stub.
\end{definition}

\begin{proposition} \label{prop:normal-derivations} %
  Let $\DDP$ be a disjunctive program, let $\Dat$ be a dataset, and
  let $\fml$ be a nonempty disjunction of facts. For every derivation
  of $\fml$ from $\DDP\cup\Dat$ there is a normal derivation of a
  nonempty subset of $\fml$ from $\DDP\cup\Dat$ or a normal derivation
  of\/ $\bot$ from $\DDP\cup\Dat$.
\end{proposition}

\begin{proof}
  Let $\rho=(T,\lambda)$ be a derivation of $\fml$ from $\DDP\cup\Dat$
  and let $v$ be the root of~$T$. We proceed by induction on the size
  of $T$. If $\top(a)\in\lambda(v)$ for some $a$, then $a$ must occur
  in $\DDP\cup\Dat$, and hence we can derive $\top(a)$ in one step
  with a rule $P(x_1,\dots,x_n)\to\top(x_i)\in\DDP_\top$ (if
  $P(a_1,\dots,a_{i-1},a,a_{i+1},\dots,a_n)\in\Dat$) or with the rule $(\to
  \top(a))\in\DDP_\top$ (if $a$ occurs in~$\DDP$).

  If $\top$ does not occur in $\lambda(v)$, we proceed as follows. Let
  $v_1,\dots,v_n$ be the successors of $v$ in $T$ ($n=0$ if $v$ is a
  leaf in $T$), let $r\in\DDP$ be the rule used to derive $\lambda(v)$
  from $\lambda(v_1),\dots,\lambda(v_n)$, and let $\sigma$ be the
  substitution used in the corresponding hyperresolution step. By the
  inductive hypothesis, for every $i\in[1,n]$ there is some $\fmm_i$
  such that $\fmm_i$ is a nonempty subset of $\lambda(v_i)$ or
  $\fmm_i=\bot$ and $\fmm_i$ has a normal derivation from
  $\DDP\cup\Dat$. W.l.o.g., let $\fmm_i\ne\bot$ for every $i\in[1,n]$
  (otherwise, the claim is immediate). We then distinguish two cases.
  If $r$ applies to $\fmm_1,\dots,\fmm_n$ with substitution $\sigma$,
  then the hyperresolvent $\fml'$ of $r$ and $\fmm_1,\dots,\fmm_n$ is a
  nonempty subset of $\fml$ ($\fml'$ is nonempty since the only rule
  in $\DDP$ with an empty head is $(\bot\to)$ but, by assumption,
  $\fml_1\ne\bot$).
  If $r$ does not apply to $\fmm_1,\dots,\fmm_n$ with substitution
  $\sigma$, the claim follows since, for some $i$, we have
  $\fmm_i\subseteq\fml$.
\end{proof}

Proposition~\ref{prop:normal-derivations} allows us to consider
only normal derivations. In the following, without loss of generality,
we assume every derivation to be normal.

\begin{proposition}\label{prop:edb-derivation-shape}
  Let $\DDP$ be a disjunctive program, let $\Dat$ be a dataset, let
  $\rho=(T,\lambda)$ be a derivation from $\DDP\cup\Dat$, and let $v$
  be a node in $T$. If $\lambda(v)$ contains an EDB atom, then
  $\lambda(v)$ is a singleton.
\end{proposition}

\begin{proof}
  The claim follows since whenever $\lambda(v)$ contains an EDB atom,
  we either have that $v$ is a leaf in $T$, and thus
  $\lambda(v)\in\Dat$, or that $v$ is the root of a $\top$-stub (since
  $\rho$ is implicitly assumed to be normal), and thus
  $\lambda(v)=\top(a)$ for some individual~$a$.
\end{proof}

\begin{definition} 
  A nonempty tree $T'=(V',E')$ is an \emph{upper portion} of a tree
  $T=(V,E)$ if the following conditions hold:
  \begin{itemize}
  \item $V'\subseteq V$ and $E'$ is the restriction of $E$ to $V'$.
  \item $T$ and $T'$ have the same root.
  \item If $v$ is an internal node in $T'$, then every child of $v$ in
    $T$ is contained in $V'$.
  \end{itemize}
  Let $\P$ be a (disjunctive) datalog program, $\Dat$ a dataset, and
  $\rho=(T,\lambda)$ a derivation of a fact $P(\ve a)$ from
  $\P\cup\Dat$ where $P\ne\top$. An \emph{upper portion} of $\rho$ is
  a pair $\rho'=(T',\lambda')$ 
 such that:
  \begin{itemize}
  \item $T'$ is an upper portion of $T$;
  \item $\lambda'$ is the restriction of $\lambda$ to the nodes in $T'$;
  \item If $\lambda'(v)=\top(b)$ for some $v\in T'$ and some
    individual $b$, then $v$ is a leaf in $T'$.\qedhere
  \end{itemize}
\end{definition}

\rewcorrectlinear*

\begin{proof}
  Analogous to the proof of Theorem~\ref{thm:rew-correct-semi} in
  Appendix~\ref{sec:pf-weak} (but simpler).
\end{proof}

\begin{restatable}{lemma}{invrewcompleteaux} \label{lem:inv-rew-complete-aux}
  Let $\P$ be a datalog program, let $\Dat$ be a dataset, let $P$ be an IDB
  predicate in $\P^e$, and let $\rho=(T,\lambda)$ be a derivation of a
  fact $P(\ve a)$ from $\P^e\cup\Dat$.
  Given a tree\/ $T'$, let $\mathsf{leaves}(T')$ be the set of leaves of\/
  $T'$ and, given a set of nodes\/ $S$, let
  $\lambda(S)=\bigcup\mset{\lambda(t)}{t\in S}$.
  For every upper portion\/ $\rho'=(T',\lambda')$ of\/ $\rho$ we have
  $\Psi(\P)\cup\Dat\models\bigvee_{Q(\ve
    b)\in\lambda(\mathsf{leaves}(T'))}Q^P(\ve b,\ve a)$.
\end{restatable}

\begin{proof}
  Let $\rho'=(T',\lambda')$ be an upper portion of $\rho$ and let $v$
  be the root of $T'$. In particular, we have $\lambda(v)=P(\ve
  a)$. We proceed by induction on the size of $T'$. If $v$ is the only
  node in $T'$, the claim reduces to $\Psi(\P)\cup\Dat\models P^P(\ve
  a,\ve a)$. This follows since $\fml_\top
  \to P^P(\ve x,\ve x)\in\Psi(\P)$ and
  $\Psi(\P)_\top\cup\Dat\models\fml_\top(\ve a)$.

  Now suppose $T'$ contains more than one node and let
  $\mathsf{leaves}(T')=\set{v_1,\dots,v_n}$. Since $T'$ is a tree, it
  has a node $w$ of height~1. W.l.o.g., let $v_1,\dots,v_k$ ($1\le
  k\le n$) be the children of $w$ in $T'$, let $r=\bigwedge_{i=1}^k
  Q_i(\ve s_i)\to R(\ve t)\in\P^e$ be the rule used to derive
  $\lambda(w)$ from $\lambda(s_1),\dots,\lambda(s_k)$, and let
  $\sigma$ be the substitution used in the corresponding
  hyperresolution step. Since $w$ is an internal node in $T'$, we have
  $R\ne\top$, and hence $r\notin\P^e_\top$. Then:
  \begin{enumerate}
  \item $S=\set{w,v_{k+1},\dots,v_n}$ is the set of leaves of an upper
    portion of $\rho$ that is strictly smaller than $\rho'$;
  \item $\lambda(w)=R(\ve t\sigma)$ and $\lambda(v_i)=Q(\ve s_i\sigma)$ for
    every $i\in[1,k]$;
  \item $(\lambda(S)\setminus\set{R(\ve t\sigma)})
    \cup\set{Q_1(\ve s_1\sigma),\dots,Q_k(\ve s_k\sigma)}
    \subseteq\lambda(\mathsf{leaves}(T'))$;
  \item $R^P(\ve t,\ve y)\to\bigvee_{i=1}^k
  Q_i^P(\ve s_i,\ve y)\in\Psi(\P)$.
  \end{enumerate}
  By (1), (2), and the inductive hypothesis,
  $\Psi(\DDP)\cup\Dat\models R^P(\ve t\sigma,\ve a)\lor\bigvee_{Q(\ve
    b)\in\lambda(\set{v_{k+1},\dots,v_n})}Q^P(\ve b,\ve a)$.  Hence by
  (3), it suffices to show $\Psi(\P)\cup\Dat\models R^P(\ve
  t\sigma,\ve a)\to\bigvee_{i=1}^k Q_i^P(\ve s_i\sigma,\ve a)$, which
  follows by~(4).
\end{proof}

\begin{restatable}{lemma}{invrewcomplete} \label{lem:inv-rew-complete}
  Let $\P$ be a datalog program. For every dataset $\Dat$ over the
  signature of\/ $\P$ and every fact $\alpha$ such that
  $\P^e\cup\Dat\models\alpha$ we have $\Psi(\P)\cup\Dat\models\alpha$.
\end{restatable}

\begin{proof}
  Let $\Dat$ be a dataset and $P(\ve a)$ a fact such that
  $\P^{e}\cup\Dat\models P(\ve a)$. W.l.o.g., let $P$ be IDB in
  $\P^e$. We show $\Psi(\P)\cup\Dat\models P(\ve a)$. By completeness
  of hyperresolution, $P(\ve a)$ has a derivation $\rho=(T,\lambda)$
  from $\P^e\cup\Dat$.  Let $\rho'=(T',\lambda')$ be the largest upper
  portion of $\rho$.  By Lemma~\ref{lem:inv-rew-complete-aux},
  $\Psi(\P)\cup\Dat\models\bigvee_{Q(\ve
    b)\in\lambda(\mathsf{leaves}(T'))}Q^P(\ve b,\ve a)$. Therefore, it
  suffices to show that $\Psi(\P)\cup\Dat\cup\set{Q^P(\ve b,\ve
    a)}\models P(\ve a)$ for every $Q(\ve
  b)\in\lambda(\mathsf{leaves}(T'))$. Since $\rho'$ is maximal, we
  distinguish the following three cases for $Q(\ve b)$:
  \begin{itemize}
  \item $Q(\ve b)\in\Dat$. Then $Q$ is EDB in $\P^e$ (since $\Dat$
    only contains facts about predicates in $\P$ and every predicate
    in $\P$ is EDB in $\P^e$). Hence $Q(\ve z)\land Q^P(\ve z,\ve
    y)\to P(\ve y)\in\Psi(\P)$. The claim follows.
  \item $Q(\ve b)$ is ground and $(\to Q(\ve
    b))\in\P^e\setminus\P^e_\top$. Hence $Q^P(\ve b,\ve
    y)\to\bot\in\Psi(\P)$, and consequently
    $\Psi(\P)\cup\Dat\cup\set{Q^P(\ve b,\ve a)}\models\bot$. The claim follows.
  \item $Q(\ve b)=\top(b)$. Then the claim follows since $\top(z)\land
    \top^P(z,\ve y)\to P(\ve y)\in\Psi(\P)$ and
    $\Psi(\P)\cup\Dat\models\top(b)$.\qedhere
  \end{itemize}
\end{proof}

\begin{definition}
  Let $\P$ be a datalog program and let $Q(\ve b)$ be a fact where $Q$ is
  IDB in $\P$.  A disjunction $\fml$ of facts is \emph{focused on}
  $Q(\ve b)$ w.r.t. $\P$ if every disjunct $\alpha\in\fml$ has one
  of the following forms:
  \begin{itemize}
  \item $\alpha=P(\ve a)$ where $P$ is EDB in $\P$;
  \item $\alpha=\bot$;
  \item $\alpha=Q(\ve b)$;
  \item $\alpha=P^Q(\ve a,\ve b)$ for some $P$ and $\ve a$.
  \end{itemize}
  Let 
  $\rho=(T,\lambda)$ be a derivation (not necessarily from $\P$).
  We call $\rho$ \emph{focused on} $Q(\ve b)$ w.r.t. $\P$ if so is the
  label of every node in~$T$.  Given a node $v\in T$, we define
  $\lambda_{\mathrm{base}}(v):=\mset{P(\ve a)}{P^Q(\ve a,\ve
    b)\in\lambda(v)}$.
\end{definition}

\begin{lemma} \label{lem:focusing}
  Let $\P$ be a datalog program and let $\Dat$ be a dataset over the
  signature of\/ $\P$.
  Every derivation 
  from $\Psi(\P)\cup\Dat$ is focused on some fact $\alpha$
  w.r.t. $\P^e$.
\end{lemma}

\begin{proof}
  Let $\rho=(T,\lambda)$ be a derivation from $\Psi(\P)\cup\Dat$ and
  let $v$ be the root of $T$. We show that $\rho$ is focused on some
  $\alpha$ by induction on the size of $T$. If $v$ is the only node in
  $T$, we distinguish the following cases:
  \begin{itemize}
  \item $\lambda(v)\in\Dat$. Then $\lambda(v)=P(\ve a)$ where $P$ is
    EDB in $\P^e$, and thus $\rho$ is focused on every IDB predicate in $\P^e$.
  \item $\arity Q=0$ and $\lambda(v)$ is obtained by a rule of the
    form $(\to\bigvee_{i=1}^n P_i^Q(\ve s_i))\in\Psi(\P)$ for some IDB
    predicate $Q$ in $\P^e$. Then $\lambda(v)=\bigvee_{i=1}^n
    P_i^Q(\ve s_i)$, meaning $\rho$ is focused on $Q$.
  \end{itemize}
  If $v$ has successors $v_1,\dots,v_m$, we distinguish the following
  cases depending on the shape of the rule $r\in\Psi(\P)$ used to
  obtain $\lambda(v)$ from $\lambda(v_1),\dots,\lambda(v_m)$.
  \begin{itemize}
  \item $r\in\Psi(\P)_\top$. Then $m=1$ and $\lambda(v)=\top(a)$ for
    some $a$. By the inductive hypothesis, $\lambda(v_1)$ is focused
    on a fact w.r.t. $\P^e$. The claim follows since $\top$ is EDB in
    $\P^e$.
  \item $r=\fml_\top\land P^Q(\ve t,\ve y)\to\bigvee_{i=1}^n R^Q(\ve
    s_i,\ve y)$. Let $\sigma$ be the substitution used in the
    hyperresolution step. W.l.o.g., let $P^Q(\ve t\sigma,\ve
    y\sigma)\in\lambda(v_1)$.  Then, by
    Proposition~\ref{prop:edb-derivation-shape},
    $\lambda(v_j)\in\fml_\top$ for $j\in[2,m]$. Hence,
    $\lambda(v)\subseteq\lambda(v_1)\cup\set{\bot}$. The claim follows
    since, by the inductive hypothesis, the subderivation rooted at
    $v_1$ is focused on $Q(\ve b)$ (so, in particular, $\ve
    y\sigma=\ve b$), and $v_2,\dots,v_m$ are focused on every IDB
    predicate in $\P^e$.
  \item $r=(\bot\to)$ or $r=\fml_\top\to\bigvee_{i=1}^n R_i^Q(\ve
    s_i,\ve b)$ for some $R$ and $\ve s_1,\dots,\ve s_n$. In both
    cases, the argument proceeds analogously to the preceding case
    (but simpler).
  \item $r=P(\ve z)\land P^Q(\ve z,\ve y)\to Q(\ve y)$. Then
    $m=2$. Let $\sigma$ be the substitution used in the
    hyperresolution step and, w.l.o.g., let $P^Q(\ve z\sigma,\ve
    y\sigma)\in\lambda(v_1)$. Then, by
    Proposition~\ref{prop:edb-derivation-shape},
    $\lambda(v_2)=P(\ve z\sigma)$. Hence,
    $\lambda(v)\subseteq\lambda(v_1)\cup\set{Q(\ve y\sigma)}$. The
    claim follows since, by the inductive hypothesis, the
    subderivation rooted at $v_1$ is focused on $Q(\ve b)$ and
    $\lambda(v_2)$ is focused on every IDB predicate
    in~$\P^e$.\qedhere
  \end{itemize}
\end{proof}

Since the root of a derivation of a fact $\alpha$ has to be labeled
with $\alpha$, Lemma~\ref{lem:focusing} implies the following
corollary.

\begin{corollary} \label{cor:focusing}
  Let $\P$ be a datalog program, let $\Dat$ be a dataset over the
  signature of\/ $\P$, and let $Q(\ve b)$ be a fact where $Q$ is IDB in
  $\P^e$. Every derivation of $Q(\ve b)$ from $\Psi(\P)\cup\Dat$ is
  focused on $Q(\ve b)$ w.r.t. $\P^e$.
\end{corollary}

\begin{restatable}{lemma}{inv-rew-sound}
  Let $\P$ be a datalog program, let $\Dat$ be a dataset over the
  signature of\/ $\P$, and let $\rho=(T,\lambda)$ be a derivation of a
  fact $Q(\ve b)$ from $\Psi(\P)\cup\Dat$, where $Q$ is IDB in
  $\P^e$. For every node $v$ in\/ $T$ whose label contains an IDB atom in
  $\Psi(\P)$, we have
  $\P^e\cup\Dat\models(\bigwedge_{\alpha\in\lambda_{\mathrm{base}}(v)}\alpha)\to
  Q(\ve b)$.
\end{restatable}

\begin{proof}
  We proceed by induction on the height of $v$ in $T$. Let
  $v_1,\dots,v_m$ be the successors of $v$ in $T$ (where $m=0$ if $v$
  is a leaf in $T$). If $Q(\ve b)\in\Dat$, the claim is vacuous since
  $\Dat$ contains only facts about predicates in $\P$, every predicate
  in $\P$ is EDB in $\P^e$, and every EDB predicate in $\P^e$ is EDB
  in $\Psi(\P)$. Otherwise, we distinguish the following cases
  depending on the shape of the rule
  $r\in\Psi(\P)\setminus\Psi(\P)_\top$ used to obtain $\lambda(v)$
  from $\lambda(v_1),\dots,\lambda(v_m)$ (by
  Corollary~\ref{cor:focusing}, we only need to consider cases that
  can occur in a derivation focused on $Q(\ve b)$ w.r.t. $\P^e$):
  \begin{itemize}
  \item $r=\fml_\top\land P^Q(\ve t,\ve y)\to\bigvee_{i=1}^n R_i^Q(\ve
    s_i,\ve y)$ such that $r'=\bigwedge_{i=1}^n R_i(\ve s_i)\to P(\ve
    t)\in\P^e$. Let $\sigma$ be the substitution used in the
    corresponding hyperresolution step and let, w.l.o.g., $P^Q(\ve
    t\sigma,\ve y\sigma)\in\lambda(v_1)$. Then, by the inductive hypothesis,
    $\P^e\cup\Dat\models(\bigwedge_{\alpha\in\lambda_{\mathrm{base}}(v_1)}\alpha)\to
    Q(\ve b)$. Moreover, $(\lambda(v_1)\setminus\set{P^Q(\ve
      t\sigma,\ve y\sigma)})\cup\bigcup_{i=1}^n R_i^Q(\ve
    s_i\sigma,\ve y\sigma)\subseteq\lambda(v)$. The claim follows
    since, by $r'$,
    $\P^e\models(\bigwedge_{\alpha\in\lambda_{\mathrm{base}}(v)}\alpha)\to
    P(\ve t\sigma)$.
  \item $r=\fml_\top\to\bigvee_{i=1}^n R_i^Q(\ve s_i,\ve y)$ such that
    $r'=\bigwedge_{i=1}^n R_i(\ve s_i)\to\bot\in\P^e$. Let $\sigma$ be the
    substitution used in the corresponding hyperresolution step. Then
    $\bigcup_{i=1}^n R_i^Q(\ve s_i\sigma,\ve
    y\sigma)\subseteq\lambda(v)$, and hence, by $r'$,
    $\P^e\models(\bigwedge_{\alpha\in\lambda_{\mathrm{base}}(v)}\alpha)\to\bot$. The
    claim follows.
  \item $r=\fml_\top\to Q^Q(\ve y,\ve y)$. Since $\rho$ is focused on
    $Q(\ve b)$, we have $\lambda(v)=Q^Q(\ve b,\ve b)$, and the claim
    ($\P^e\cup\Dat\models Q(\ve b)\to Q(\ve b)$) is immediate.
  \item $r=P(\ve z)\land P^Q(\ve z,\ve y)\to Q(\ve y)$ for some EDB
    predicate $P$ in $\P^e$. Let $\sigma$ be the substitution used in
    the corresponding hyperresolution step (in particular, $\ve
    y\sigma=\ve b$).  Let, w.l.o.g., $P^Q(\ve z\sigma,\ve
    b)\in\lambda(v_1)$ and $\lambda(v_2)=P(\ve z\sigma)$
    (Proposition~\ref{prop:edb-derivation-shape}). Since $P$ is
    EDB in $\P^e$ and hence in $\Psi(\P)$, we have $P(\ve
    z\sigma)\in\Dat$. By the inductive hypothesis,
    $\P^e\cup\Dat\models(\bigwedge_{\alpha\in\lambda_{\mathrm{base}}(v_1)}\alpha)\to
    Q(\ve b)$, and therefore 
    $\P^e\cup\Dat\models(\bigwedge_{\alpha\in\lambda_{\mathrm{base}}(v_1)\setminus\set{P(\ve
        z\sigma)}}\alpha)\to Q(\ve b)$.  The claim follows since
    $\lambda(v_1)\setminus\set{P^Q(\ve z\sigma,\ve
      b)}\subseteq\lambda(v)$.\qedhere
  \end{itemize}
\end{proof}

By completeness of hyperresolution and the observation that
$\lambda_{\mathrm{base}}(v)=\emptyset$ whenever $\lambda(v)$ contains
no facts of the form $P^Q(\ve a)$, we obtain the following corollary.

\begin{corollary} \label{cor:inv-rew-sound} %
  Let $\P$ be a datalog program. For every dataset $\Dat$ over the
  signature of\/ $\P$ and every atom $\alpha$ over the signature of\/
  $\P^e$ such that $\Psi(\P)\cup\Dat\models\alpha$ we have
  $\P^e\cup\Dat\models\alpha$.
\end{corollary}

\invrewcorrect*
\begin{proof}
  By construction, $\Psi(\P)$ is a linear disjunctive program of size
  quadratic in the size of $\P$.
  Since $\P^e$ is a rewriting of $\P$, to prove that $\Psi(\P)$ is a
  rewriting of $\P$ it suffices to show that
  $\eval{\P^e}\Dat|_S=\eval{\Psi(\P)}\Dat|_S$ for every dataset $\Dat$
  over the signature of $\P$ and every set $S$ of predicates
  in~$\P^e$. This follows by Lemma~\ref{lem:inv-rew-complete} and
  Corollary~\ref{cor:inv-rew-sound}.
%
\end{proof}


%% file: proofs-weaklinearity.tex
\section{Proofs for Section~\ref{sec:weak}} \label{sec:pf-weak}

We begin by generalising Proposition~\ref{prop:edb-derivation-shape} as follows.

\begin{proposition}\label{prop:datalog-derivation-shape}
  Let $\DDP$ be a disjunctive program, let $\Dat$ be a dataset, let
  $\rho=(T,\lambda)$ be a derivation from $\DDP\cup\Dat$, and let $v$
  be a node in $T$.
  If $\lambda(v)$ contains a datalog atom, then $\lambda(v)$ is a singleton.
\end{proposition}

\begin{proof}
  Straightforward induction on the height of $v$ in $\rho$. The case
  where $\lambda(v)$ contains an atom of the form $\top(a)$ follows by
  the implicit assumption that $\rho$ is normal.
\end{proof}

\begin{proposition} \label{prop:datalog-entailment-pres} %
  Let $\DDP$ be a disjunctive program, let $Q$ be a datalog predicate
  in $\DDP$, and let $\DDP_Q$ be the set of all rules in $\DDP$ on
  which $Q$ depends. For every dataset $\Dat$ and vector of individuals $\ve a=a_1\dots a_{\arity{Q}}$:
  $\DDP\cup\Dat\vdash Q(\ve a)$ \,iff\, $\DDP_Q\cup\Dat\vdash Q(\ve a)$.
\end{proposition}

\begin{proof}
  The inclusion from right to left is immediate. The inclusion from
  left to right follows by 
  a straightforward induction on the derivation of a fact $Q(\ve a)$
  from $\DDP\cup\Dat$ exploiting the implicit normality assumption.
\end{proof}

Since datalog predicates only depend on rules that contain no
disjunctive predicates, and a WL program $\DDP$ coincides
with $\Xi'(\DDP)$ on rules that contain no disjunctive predicates, we
obtain (by correctness of hyperresolution):

\begin{corollary} \label{cor:datalog-entailment-pres} %
  Let $\DDP$ be a WL program and let $Q$ be a
  datalog predicate in $\DDP$. For every dataset $\Dat$
  and vector of individuals $\ve a=a_1\dots a_{\arity{Q}}$:
  $\DDP\cup\Dat\models Q(\ve a)$ \,iff\, $\Xi'(\DDP)\cup\Dat\models Q(\ve a)$.
\end{corollary}

\begin{restatable}{lemma}{rew-completene-aux}
  \label{lem:completeness-aux} %
  Let $\DDP$ be a WL program, let $\Dat$ be a
  dataset, let $P$ be a disjunctive predicate in $\DDP$, and let
  $\rho=(T,\lambda)$ be a derivation of a fact $P(\ve a)$ from
  $\DDP\cup\Dat$. Then for every node $v\in T$ in an upper portion of
  $\rho$ and every disjunct $Q(\ve b)\in\lambda(v)$ where $Q$ is
  disjunctive in $\DDP$, we have $\Xi'(\DDP)\cup\Dat\models Q^P(\ve
  b,\ve a)$.
\end{restatable}

\begin{proof}
  Let $v\in T$ and $Q(\ve b)\in\lambda(v)$ be as required. We show the
  claim by induction on the distance of $v$ from the root of $T$. If
  $v$ is the root of $T$, then $Q(\ve b)=P(\ve a)$ and the claim
  ($\Xi'(\DDP)\cup\Dat\models P^P(\ve a,\ve a)$) follows since
  $\top(y_1)\land\dots\land\top(y_{\arity{P}})\to P^P(\ve y,\ve
  y)\in\Xi'(\DDP)$ and $\Xi'(\DDP)_\top\cup\Dat\models\top(a_i)$ for
  every $a_i\in\ve a$.

  If $v$ is not the root of $T$, then it must have a predecessor $w$
  and siblings $v_1,\dots,v_n$ ($n\ge 0$) in $T$ such that either (a)
  $Q(\ve b)\in\lambda(w)$ or (b) $\lambda(w)$ is a hyperresolvent of
  $\lambda(v),\lambda(v_1),\dots,\lambda(v_n)$ and some rule $Q(\ve
  s)\land\bigwedge_{i=1}^n R_i(\ve s_i)\to\bigvee_{j=1}^m S_j(\ve
  t_j)\in\DDP\setminus\DDP_\top$, where the atom $Q(\ve s)$ is resolved with
  $\lambda(v)$ and the atoms $R_i(\ve s_i)$ are resolved with
  $\lambda(v_i)$. If $Q(\ve b)\in\lambda(w)$, the claim follows by the
  inductive hypothesis so, w.l.o.g., suppose we are in Case
  (b). Since, by assumption, $Q$ is disjunctive and $\DDP$ is WL,
  all $R_i$ are datalog. Hence, $\Xi'(\DDP)$ contains a rule
  $r=\fml_\top 
  \land(\bigwedge_{j=1}^m S_j^P(\ve
  t_j,\ve y))\land\bigwedge_{i=1}^n R_i(\ve s_i)\to Q^P(\ve s,\ve y)$.
  Let $\sigma$ be the substitution used in the hyperresolution step
  deriving $\lambda(w)$. Then $\ve s\sigma=\ve b$,
  $\lambda(v_i)=R_i(\ve s_i\sigma)$ for every $i\in[1,n]$ (by
  Proposition~\ref{prop:datalog-derivation-shape}), and
  $\bigvee_{j=1}^m S_j(\ve t_j\sigma)\subseteq\lambda(w)$. By the
  inductive hypothesis, we then have $\Xi'(\DDP)\cup\Dat\models
  S_j^P(\ve t_j\sigma,\ve a)$ for every $j\in[1,m]$. Moreover, by
  Corollary~\ref{cor:datalog-entailment-pres},
  $\Xi'(\DDP)\cup\Dat\models R_i(\ve s_i\sigma)$ for every
  $i\in[1,n]$. Finally, we have 
  $\Xi'(\P)_\top\cup\Dat\models\fml_\top\sigma$. The claim follows with $r$.
\end{proof}

\begin{restatable}{lemma}{rew-complete} \label{lem:rew-complete} %
  Let $\DDP$ be a WL program. For every dataset
  $\Dat$ 
  and every fact $\alpha$ such that $\DDP\cup\Dat\models\alpha$ we
  have $\Xi'(\DDP)\cup\Dat\models\alpha$.
\end{restatable}

\begin{proof}
  Let $\DDP\cup\Dat\models P(\ve a)$. We show that
  $\Xi'(\DDP)\cup\Dat\models P(\ve a)$. W.l.o.g., $P(\ve a)\notin\Dat$
  (otherwise, the claim is trivial) and $P$ is disjunctive (otherwise,
  the claim follows by
  Corollary~\ref{cor:datalog-entailment-pres}). By completeness of
  hyperresolution, there is a derivation $\rho=(T,\lambda)$ of $P(\ve
  a)$ from $\DDP\cup\Dat$. Since $P(\ve a)\notin\Dat$ and $P$ is
  disjunctive, 
  there is an upper portion $\rho'$ of $\rho$ and a node $v$ in
  $\rho'$ such that:
  \begin{enumerate}
  \item $\lambda(v)$ contains a disjunctive predicate;
  \item $v$ has no successor $w$ in $T$ such that $\lambda(w)$
    contains a disjunctive predicate.
  \end{enumerate}
  We distinguish two cases. If 
  $\lambda(v)\in\Dat$, then
  $\lambda(v)=Q(\ve b)$ for some $Q$ and $\ve b$. By
  Lemma~\ref{lem:completeness-aux}, we have $\Xi'(\DDP)\cup\Dat\models
  Q^P(\ve b,\ve a)$. The claim follows since $Q(\ve z)\land Q^P(\ve
  z,\ve y)\to P(\ve y)\in\Xi'(\DDP)$.

  If $\lambda(v)\notin\Dat$, then 
  $v$ has successors $v_1,\dots,v_n$ ($n\ge 0$)
  in $T$ such that $\lambda(v)$ is a hyperresolvent of
  $\lambda(v_1),\dots,\lambda(v_n)$ and a rule in
  $\DDP\setminus\DDP_\top$ of the form $\bigwedge_{i=1}^n R_i(\ve
  s_i)\to\bigvee_{j=1}^m S_j(\ve t_j)$, where the atoms $R_i(\ve s_i)$
  are resolved with $\lambda(v_i)$. Since, by assumption, all $R_i$
  are datalog, $\Xi'(\DDP)$ contains a rule $r=(\bigwedge_{j=1}^m
  S_j^P(\ve t_j,\ve y))\land\bigwedge_{i=1}^n R_i(\ve s_i)\to P(\ve
  y)$. Let $\sigma$ be the substitution used in the hyperresolution
  step deriving $\lambda(v)$. By Lemma~\ref{lem:completeness-aux}, we
  then have $\Xi'(\DDP)\cup\Dat\models S_j^P(\ve t_j\sigma,\ve a)$ for
  every $j\in[1,m]$. By
  Proposition~\ref{prop:datalog-derivation-shape}, we have
  $\lambda(v_i)=R_i(\ve s_i\sigma)$, and hence, by
  Corollary~\ref{cor:datalog-entailment-pres},
  $\Xi'(\DDP)\cup\Dat\models R_i(\ve s_i\sigma)$ for every
  $i\in[1,n]$. The claim follows with~$r$.
\end{proof}

\begin{restatable}{lemma}{rew-sound} \label{lem:rew-sound} %
  Let $\DDP$ be a WL program, let $\Dat$ be a dataset
  over the signature of $\DDP$, and let $\rho=(T,\lambda)$ be a
  derivation of a fact $\alpha$ from $\Xi'(\DDP)\cup\Dat$ where
  $\alpha$ is not of the form $\top(a)$. Then:
  \begin{enumerate}
  \item For every $v\in T$, $\lambda(v)=P(\ve a)$ where $P$ occurs in
    $\DDP$, or $\lambda(v)=P^Q(\ve a,\ve b)$ where $P,Q$ are
    disjunctive in $\DDP$.
  \item If $\alpha=P(\ve a)$ where $P$ occurs in $\DDP$, then
    $\DDP\cup\Dat\models P(\ve a)$.
  \item If $\alpha=P^Q(\ve a,\ve b)$, then
    $\DDP\cup\Dat\models P(\ve a)\to Q(\ve b)$.
  \end{enumerate}
\end{restatable}

\begin{proof}
  We begin by showing (1). Since $\Xi'(\DDP)$ is datalog, $\lambda(v)$
  contains only one atom for every $v\in T$. The claim follows since
  $\Dat$ contains only predicates in $\DDP$ and
  the rules of $\Xi'(\DDP)$ can only infer facts of the form $P(\ve a)$
  where $P$ occurs in $\DDP$ or $P^Q(\ve a,\ve b)$ where $P,Q$ are
  disjunctive in $\DDP$.

  We now show (2) and (3) by simultaneous induction on the height $n$
  of $T$. If $n=0$, we distinguish three cases:
  \begin{itemize}
  \item $\alpha\in\Dat$. Then $\Dat\models\alpha$ and the claim is immediate.
  \item $\alpha=P(\ve a)$ where $P$ is datalog in $\DDP$ and $r=(\to
    P(\ve a))\in\Xi'(\DDP)$. Then $r\in\DDP$ and the claim is immediate.
  \item $\alpha=\bot^Q$, $\arity Q=0$, and $(\to
    \bot^Q)\in\Xi'(\DDP)$. Then, since $(\bot\to)\in\DDP$, we have $\DDP\models\bot\to
    Q$ for every predicate~$Q$.
  \item $\alpha=P^P$ where $P$ is disjunctive in $\DDP$ and
    $\arity P=0$. The claim ($\DDP\cup\Dat\models P\to P$) is
    immediate.
  \end{itemize}
  If $n>0$, the root $v$ of $T$ has children $v_1,\dots,v_n$ and $\alpha$
  is a hyperresolvent of $\lambda(v_1),\dots,\lambda(v_n)$ and a rule
  $r\in\Xi'(\DDP)\setminus\Xi'(\DDP)_\top$. We distinguish five cases:
  \begin{itemize}
  \item $r$ contains no disjunctive predicates. Then $\alpha$ is a
    datalog atom and the claim follows by
    Corollary~\ref{cor:datalog-entailment-pres}.
  \item $r=\fml_\top
    \to P^P(\ve y,\ve y)$
    where $P$ is disjunctive in $\DDP$. Then $\alpha=P^P(\ve a,\ve a)$ for
    some $\ve a$, and the claim ($\DDP\cup\Dat\models P(\ve a)\to P(\ve
    a)$) is immediate.
  \item $r=Q(\ve z)\land Q^P(\ve z,\ve y)\to P(\ve y)$. Then
    $\alpha=P(\ve a)$ for some $\ve a$. By the
    Corollary~\ref{cor:datalog-entailment-pres}, we have
    $\DDP\cup\Dat\models Q(\ve b)$, and by the inductive hypothesis,
    $\DDP\cup\Dat\models Q(\ve b)\to P(\ve a)$ for some $\ve b$. Hence
    $\DDP\cup\Dat\models P(\ve a)$.
  \item $r=\fml_\top
    \land\fml\land\bigwedge_{i=1}^n R_i^Q(\ve s_i,\ve y)\to P^Q(\ve
    t,\ve y)$ where $\fml$ is the conjunction of all datalog atoms in
    $r$ and $r'=\fml\land P(\ve t)\to\bigvee_{i=1}^n R_i(\ve
    s_i)\in\DDP$. Then $\alpha=P^Q(\ve a,\ve b)$ for some $\ve a$ and $\ve
    b$. By Corollary~\ref{cor:datalog-entailment-pres}, for every $i\in[1,n]$ there is
    some $\ve c_i$ such that $\DDP\cup\Dat\models\Subst{\ve a\ve b\ve
      c_1\dots\ve c_n}{\ve t\ve y\ve s_1\dots\ve s_n}\fml$, and by the inductive hypothesis,
    $\DDP\cup\Dat\models R_i(\ve c_i)\to Q(\ve b)$. With $r'$, we
    obtain $\DDP\cup\Dat\models P(\ve a)\to Q(\ve b)$.
  \item $r=\fml\land\bigwedge_{i=1}^n R_i^P(\ve s_i,\ve y)\to P(\ve
    y)$ where $\fml$ is the conjunction of all datalog atoms in $r$
    and $r'=\fml\to\bigvee_{i=1}^n R_i(\ve s_i)\in\DDP$. Then
    $\alpha=P(\ve a)$ for some $\ve a$. By
    Corollary~\ref{cor:datalog-entailment-pres}, for every $i\in[1,n]$
    there is some $\ve b_i$ such that $\DDP\cup\Dat\models\Subst{\ve
      a\ve b_1\dots\ve b_n}{\ve y\ve s_1\dots\ve s_n}\fml$, and by the
    inductive hypothesis, $\DDP\cup\Dat\models R_i(\ve b_i)\to P(\ve
    a)$. With $r'$, we obtain $\DDP\cup\Dat\models P(\ve a)$.\qedhere
  \end{itemize}
\end{proof}

By completeness of hyperresolution (and
Corollary~\ref{cor:datalog-entailment-pres} for facts of the form
$\top(a)$), Lemma~\ref{lem:rew-sound}(2) implies:

\begin{corollary} \label{cor:rew-sound} %
  Let $\DDP$ be a WL program. For every dataset
  $\Dat$ and atom $\alpha$ over the signature of $\DDP$
  such that $\Xi'(\DDP)\cup\Dat\models\alpha$ we have
  $\DDP\cup\Dat\models\alpha$.
\end{corollary}

\rewcorrect*
\begin{proof}
  By construction, $\Xi'(\DDP)$ is a datalog program of
  size quadratic in the size of $\DDP$.
  Correctness of the transformation (i.e., $\Xi'(\DDP)$ being a
  rewriting of $\DDP$) follows with Lemma~\ref{lem:rew-complete} and
  Corollary~\ref{cor:rew-sound}.
%
\end{proof}

\rewcorrectsinglequery*

\begin{proof}
  Follows analogously to Theorem~\ref{thm:rew-correct-semi} with minor
  adaptations of the relevant lemmas.
\end{proof}


%% file: proofs-procedure.tex
\section{Proofs for Section~\ref{sec:procedure}}

\begin{definition}
  Let $r=\alpha\land\fml_r\to\fmm_r$ and $s=\fml_{s}\to \beta\lor\fmm_s$ be
  rules such that atom $\alpha$ is unifiable with $\beta$ with MGU $\theta$.
  The \emph{elementary unfolding} $\mathsf{ElemUnfold}(r,\alpha,s,\beta)$ 
  of $r$ at $\alpha$ using $s$ at $\beta$ is the pair
    $((\fml_r\land\fml_s\to\fmm_r\lor\fmm_s)\theta, \theta)$.
\end{definition}

\begin{procedure}[t]
\caption{$\mathsf{Unfold}$}\label{alg:unfold}
\begin{footnotesize}
    \textbf{Input:}
    $\DDP$: a disjunctive program; 
    $r$: a rule; $\alpha$: a body atom of $r$ \\
    \textbf{Output:}
    the unfolding of $r$ at $\alpha$ by $\DDP$
    \begin{algorithmic}[1]
      \State $S_0:=\mset{(s,\beta)}{s\in\DDP,\beta\textup{ a head atom in }s\textup{ unifiable with }\alpha}$
      \State $i:=0$
      \Repeat
      \State $S_{i+1}:=\emptyset$
      \For{\textbf{each} $(s,\beta)\in S_i$}
      \State $(s',\theta):=\mathsf{ElemUnfold}(r,\alpha,s,\beta)$
      \State $S_{i+1}:=S_{i+1}\cup\mset{(s',\beta'\theta)}{(s,\beta')\in S_i,\beta\neq\beta'}$
      \EndFor
      \State $i:=i+1$
      \Until{$S_i\neq\emptyset$}
      \State \Return $(\DDP\setminus\set{r})\cup\mset{s}{(s,\beta)\in S_j,\text{ for }1\le j<i}$
    \end{algorithmic}
\end{footnotesize}
\end{procedure}

An elementary unfolding step thus amounts to resolving the
relevant rules over the given predicates.
\emph{Unfolding} is then a transformation that 
allows us to replace a rule in a program with 
a sequence of elementary unfoldings in such a way that
equivalence is preserved. 
\begin{definition}
Let $\DDP$ be a disjunctive program, let $r \in \DDP$ and let $\alpha$
be a body atom in $r$; then, the \emph{unfolding of $r$ at $\alpha$
in $\DDP$}, denoted $\mathsf{Unfold}(\DDP,r,\alpha)$,
is the result of applying Procedure~\ref{alg:unfold}
to $\DDP$, $r$, and $\alpha$.
\end{definition}

\begin{restatable}{proposition}{unfolding-sound} \label{prop:unfolding-sound}
  Let $\DDP$ be a disjunctive program, $r$ a rule in $\DDP$, and $\alpha$ a
  body atom of $r$. Then $\DDP\models\mathsf{Unfold}(\DDP,r,\alpha)$.
\end{restatable}

\begin{proof}
  The claim follows by soundness of resolution since every clause in
  $\mathsf{Unfold}(\DDP,r,\alpha)\setminus\DDP$ is obtained by
  resolution from clauses in~$\DDP$.
\end{proof}


\begin{lemma} \label{lem:unfolding-complete-aux} %
  Let $\DDP$ be a disjunctive program, let
  $r=\bigwedge_{i=1}^n\alpha_i\to\fmm$ ($n\ge 1$) be a rule in $\DDP$
  where $\alpha_1$ is IDB in $\DDP$, let $\Dat$ be a dataset
  containing no occurrences of IDB predicates in $\DDP$, and let
  $\sigma$ be a ground substitution. If
  $\mathsf{Unfold}(\DDP,r,\alpha_1)\cup\Dat\models\alpha_i\sigma\lor\fmn_{\alpha_i}$
  for every $i\in[1,n]$ (where each $\fmn_{\alpha_i}$ is a ground
  disjunction of facts), then
  $\mathsf{Unfold}(\DDP,r,\alpha_1)\cup\Dat\models\fmm\sigma\lor\fmn_{\alpha_1}\lor\dots\lor\fmn_{\alpha_n}$.
\end{lemma}

\begin{proof}
  For every $i\in[1,n]$, let
  $\mathsf{Unfold}(\DDP,r,\alpha_1)\cup\Dat\models\alpha_i\sigma\lor\fmn_{\alpha_i}$. Let
  $\rho=(T,\lambda)$ be a derivation of
  $\alpha_1\sigma\lor\fmn'_{\alpha_1}$ from
  $\mathsf{Unfold}(\DDP,r,\alpha_1)\cup\Dat$ for some
  $\fmn'_{\alpha_1}\subseteq\fmn_{\alpha_1}$ (existence of $\rho$
  follows by completeness of hyperresolution). Let $s$ be the rule
  used to derive the label of the root $v$ of $\rho$ (i.e.,
  $\alpha_1\sigma\lor\fmn'_{\alpha_1}$) from the labels of its
  children $v_1,\dots,v_m$ ($s$ must exist since $\alpha_1$ is an IDB
  predicate and hence, by assumption, $\alpha_1\notin\Dat$), and let
  $\tau$ be substitution used in the corresponding hyperresolution step. Then
  $s=\bigwedge_{j=1}^m\beta_j\to\alpha'_1\lor\dots\lor\alpha'_l\lor\fmm_{\alpha'}$
  such that $\lambda(v_j)=\beta_j\tau\lor\fmn_{\beta_j}$ for every
  $j\in[1,m]$, $\alpha_1\sigma=\alpha'_1\tau=\dots=\alpha'_l\tau$
  and
  $\fmn'_{\alpha_1}=\fmm_{\alpha'}\tau\lor\fmn_{\beta_1}\lor\dots\lor\fmn_{\beta_m}$.
  Let $r_1$ be the rule obtained by elementary unfolding of $r$ at
  $\alpha_1$ using $s$ at $\alpha'_1$, and let $r_k$ ($2\le k\le l$)
  be the rule obtained by elementary unfolding of $r$ at $\alpha_1$
  using $r_{k-1}$ at $\alpha'_k$. Then
  $(\bigwedge_{j=1}^m\beta_j\tau)\land(\bigwedge_{i=2}^n\alpha_i\sigma)$
  is a substitution instance of the body of $r_l$ and
  $\fmm_{\alpha'}\tau\lor\fmm\sigma$ is the corresponding instance
  of the head of $r_l$. Hence, the claim follows from the assumption
  ($\mathsf{Unfold}(\DDP,r,\alpha_1)\cup\Dat\models\alpha_i\sigma\lor\fmn_{\alpha_i}$
  for every $i\in[2,n]$) and soundness of hyperresolution (which
  implies
  $\mathsf{Unfold}(\DDP,r,\alpha_1)\cup\Dat\models\beta_j\tau\lor\fmn_{\beta_j}$
  for every $j\in[1,m]$) with $r_l$: we obtain
  $\mathsf{Unfold}(\DDP,r,\alpha_1)\cup\Dat\models\fmm_{\alpha'}\tau\lor\fmm\sigma\lor\fmn_{\beta_1}\lor\dots\lor\fmn_{\beta_m}\lor\fmn_{\alpha_2}\lor\dots\lor\fmn_{\alpha_n}=\fmm\sigma\lor\fmn'_{\alpha_1}\lor\fmn_{\alpha_2}\lor\dots\lor\fmn_{\alpha_n}\subseteq\fmm\sigma\lor\fmn_{\alpha_1}\lor\dots\lor\fmn_{\alpha_n}$.
\end{proof}

\begin{restatable}{lemma}{unfolding-complete}
  \label{lem:unfolding-complete} %
  Let $\DDP$ be a disjunctive program, let $\Dat$ be a dataset
  containing no occurrences of IDB predicates in $\DDP$, let
  $r$ 
  be a rule in $\DDP$, and let $\alpha$ be an IDB body atom of $r$.
  For every disjunction of facts $\fml$ such that
  $\DDP\cup\Dat\vdash\fml$ we have
  $\mathsf{Unfold}(\DDP,r,\alpha)\cup\Dat\models\fml$.
\end{restatable}

\begin{proof}
  Let $\rho=(T,\lambda)$ be a derivation of $\fml$ from $\DDP\cup\Dat$
  and let $v$ be the root of $T$. We proceed by induction on the size
  of~$T$. If $\lambda(v)\in\Dat$, the claim is immediate. Otherwise,
  let $v_1,\dots,v_n$ be the successors of $v$ in $T$ ($n=0$ if $v$ is
  a leaf in $T$). By the inductive hypothesis, we have
  $\mathsf{Unfold}(\DDP,r,\alpha)\cup\Dat\models\lambda(v_i)$ for
  every $i\in[1,n]$. We distinguish two cases, depending on the rule
  $s\in\DDP$ used to derive $\lambda(v)$ from
  $\lambda(v_1),\dots,\lambda(v_n)$. If $s\ne r$, we have
  $s\in\mathsf{Unfold}(\DDP,r,\alpha)$, and the claim follows
  with~$s$. If $s=r$, the claim follows by
  Lemma~\ref{lem:unfolding-complete-aux}.
\end{proof}

\unfoldingcorrect*

\begin{proof}
  Since $\P$ is a rewriting of $\P_0$, for the first claim it suffices
  to show $\eval\P\Dat=\eval{\mathsf{Unfold}(\DDP,r,\alpha)}\Dat$ for
  datasets $\Dat$ over the signature of $\P_0$. This follows by
  Proposition~\ref{prop:unfolding-sound} and
  Lemma~\ref{lem:unfolding-complete} since $\P_0$ contains no
  occurrences of IDB predicates in $\P$, and hence neither do datasets
  over the signature of $\P_0$. The second claim is immediate since
  all rules in $\mathsf{Unfold}(\DDP,r,\alpha)\setminus\DDP$ are
  obtained by resolution from those in $\DDP$, and the set of IDB
  predicates in a program is closed under resolution.
\end{proof}


%% file: proofs-DLs.tex
\section{Proofs for Section \ref{sec:DLs}}

\begin{restatable}{proposition}{RLorComplexity} \label{prop:RLor}
Checking $\Ont \cup \Dat \models \alpha$ for 
$\Ont$ an $\RLor$ ontology, $\Dat$  a dataset, and 
$\alpha$  a fact
is \coNP-complete. 
\end{restatable}


\begin{proof}
Membership in $\coNP$ follows from the fact that 
both the rules in Table \ref{tab:RL} and
the rules axiomatising equality and $\top$
contain a bounded number of variables and atoms; hence,
the corresponding programs can be grounded in 
polynomial time and entailment in the
resulting propositional program can be checked in \coNP.
For hardness, it suffices to provide a straightforward
encoding of non-3-colorability. The following DL ontology $\mathcal{O}$
can be normalised into an $\RLor$ ontology
\begin{align*}
V&\sqsubseteq  R \sqcup G \sqcup B
& B &\sqcap \exists \mathsf{edge}.B \sqsubseteq \bot
& B &\sqcap G \sqsubseteq \bot\\
&&G &\sqcap \exists \mathsf{edge}.G \sqsubseteq \bot
& G &\sqcap R \sqsubseteq \bot \\
&&R &\sqcap \exists \mathsf{edge}.R \sqsubseteq \bot
& B &\sqcap R \sqsubseteq \bot
\end{align*}
Given an undirected graph $G = (V,E)$, the dataset
$\Dat_{G}$ contains a fact $V(a)$ for each node $a \in V$ and
facts $\mathsf{edge}(a,b)$ and $\mathsf{edge}(b,a)$ for each edge
connecting $a$ and $b$ in $E$. Then, $G$ is non-3-colorable iff
$\mathcal{O} \cup \Dat_G$ is unsatisfiable.
\end{proof}

\semiRLorComplexity*

\begin{proof}
  Hardness follows directly from the fact that the problem is already
  \textsc{PTime}-hard if $\Ont$ is an OWL 2 RL ontology; thus, we
  focus on proving membership in \textsc{PTime}.  By Theorem
  \ref{thm:rew-correct} we have that $\Ont \cup \Dat \models \alpha$
  iff $\Xi(\Ont) \cup \Dat \models \alpha\theta$ for some injective
  predicate renaming $\theta$.  Thus, it suffices to show that the
  evaluation of $\Xi(\Ont)$ over $\Dat$ can be computed in polynomial
  time in the size of $\Ont$ and $\Dat$.  First, $\Xi(\Ont)$ is of
  size at most quadratic in the size of $\Ont$, and the arity of a
  predicate in $\Xi(\Ont)$ is at most double the arity of a predicate
  in $\Ont$.  As we can see in Table \ref{tab:RL}, the rules
  corresponding to Axioms 1-11 contain a bounded number of variables
  and atoms in the body, and hence the number of variables in the body
  of each rule and the arity of predicates in $\Xi(\Ont)$ is bounded
  as well, as required.
\end{proof}

Let us fix an arbitrary $\SHIQ$ ontology $\Ont$, and let
$\Omega_{\Ont}$ be obtained from $\Ont$ by first removing all axioms
of the form 5 and then adding the relevant axioms to preserve fact
entailment as described in \cite{grauIJCAI13}. Furthermore, let us
denote with $\mathcal{R}_{\Ont}$ the subset of all axioms in $\Ont$ of
the form 5, 4, and 8. Finally, let $\mathsf{DD}(\Omega_{\Ont})$ be the
result of applying the algorithm in~\cite{hms07reasoning} to
$\Omega_{\Ont}$.
The following lemma summarises the results 
in \cite{hms07reasoning,grauIJCAI13} that are relevant
to us.

\begin{lemma}\label{lem:properties-previous}
The following properties hold:
\begin{enumerate}
  \item $\Omega_{\Ont}$ is a model conservative
extension of $\Ont$.
\item $\Omega_{\Ont}\models\mathsf{DD}(\Omega_{\Ont})$.
\item For each dataset $\Dat$ and each
fact $\alpha$ the following holds:
\begin{align}
\Ont \cup \Dat \models \alpha \text{ ~ iff ~ } &
\Omega_{\Ont} \cup \eval{\mathcal{R}_{\Ont}}{\Dat} \models \alpha  \label{eq:Omega}\\
\Omega_{\Ont} \cup \Dat \models \alpha  \text{ ~ iff ~ } &  \mathsf{DD}(\Omega_{\Ont}) \cup \Dat \models \alpha  \label{eq:DD}
\end{align}
\end{enumerate}
\end{lemma}

Then, the following theorem states that the program
obtained from $\Omega_{\Ont}$ by applying the algorithm in
\cite{hms07reasoning} and then adding the rules in
$\mathcal{R}_{\Ont}$ entails the same facts as $\Ont$ w.r.t.\ all
datasets.

\begin{restatable}{theorem}{reductionDD} \label{thm:reduction-DD}
   $\Ont \cup \Dat \models \alpha \text{ iff }
   \mathsf{DD}(\Omega_{\Ont}) \cup \mathcal{R}_{\Ont} \cup \Dat \models
   \alpha$, for every dataset $\Dat$ and fact $\alpha$ about
  individuals in $\Dat$.
\end{restatable}


\begin{proof}
  Assume that $\Ont \cup \Dat \models \alpha$. By Lemma
  \ref{lem:properties-previous}, Condition \eqref{eq:Omega} we have
  $\Omega_{\Ont} \cup \eval{\mathcal{R}_{\Ont}}{\Dat} \models \alpha$.
  Since $\eval{\mathcal{R}_{\Ont}}{\Dat}$ is a dataset, by Lemma
  \ref{lem:properties-previous}, Condition \eqref{eq:DD} we also have
  $\mathsf{DD}(\Omega_{\Ont}) \cup \eval{\mathcal{R}_{\Ont}}{\Dat}
  \models \alpha$, which then implies $\mathsf{DD}(\Omega_{\Ont}) \cup
  \mathcal{R}_{\Ont} \cup \Dat \models \alpha$, as required.

  Assume that $\Ont \cup \Dat \not\models \alpha$. Since, by Lemma
  \ref{lem:properties-previous}, $\Omega_{\Ont}$ is a conservative
  extension of $\Ont$ and $\mathcal{R}_{\Ont} \subseteq \Ont$ we have
  that $\Omega_{\Ont} \cup \mathcal{R}_{\Ont} \cup\Dat\not\models
  \alpha$.  Again, by Lemma \ref{lem:properties-previous}, we have that
  $\Omega_{\Ont} \models \mathsf{DD}(\Omega_{\Ont})$ and hence
  $\mathsf{DD}(\Omega_{\Ont}) \cup \mathcal{R}_{\Ont} \cup \Dat
  \not\models \alpha$.
\end{proof}

We define a program $\P$ to be a \emph{rewriting of an ontology}
$\Ont$ if $\P$ is a rewriting of $\mathsf{DD}(\Omega_{\Ont}) \cup
\mathcal{R}_{\Ont}$.
%
By Theorems \ref{thm:reduction-DD} and \ref{th:rewrite-procedure}, we
then obtain the following.
\begin{theorem}\label{thm:reduction-main}
  Let\/ $\Ont$ be an ontology. If\/ $\mathsf{Rewrite}$ terminates on\/
  $\mathsf{DD}(\Omega_{\Ont}) \cup \mathcal{R}_{\Ont}$ with a datalog
  program\/ $\P$, then\/ $\P$ is a rewriting of\/ $\Ont$.
\end{theorem}

\begin{proof}
  Immediate by Theorems \ref{thm:reduction-DD} and \ref{th:rewrite-procedure}.
\end{proof}
%